\documentclass{bytedance_seed}
\usepackage{amsmath,amsfonts,bm}

\def\eqref#1{equation~\ref{#1}}

\def\1{\bm{1}}

\DeclareMathAlphabet{\mathsfit}{\encodingdefault}{\sfdefault}{m}{sl}
\SetMathAlphabet{\mathsfit}{bold}{\encodingdefault}{\sfdefault}{bx}{n}

\usepackage{url}
\usepackage{amsmath}
\usepackage{amssymb}
\usepackage{amsthm}

\newtheorem{lemma}{Lemma}
\newcommand{\aname}{PACE}

\usepackage{enumitem}
\usepackage{algorithm}
\usepackage{algpseudocode}

\usepackage{soul}
\usepackage{arydshln} \usepackage{listings}
\usepackage{colortbl}

\definecolor{PromptRole}{HTML}{245A81}
\definecolor{PromptTag}{HTML}{16817A}
\definecolor{PromptText}{HTML}{30343B}

\lstdefinestyle{retoolprompt}{
    basicstyle=\fontsize{8}{10}\selectfont\ttfamily
               \color{PromptText},
    columns=fullflexible,
    keepspaces=true,
    showstringspaces=false,
    breaklines=true,
    breakatwhitespace=true,
    breakindent=1em,
    tabsize=2,
    aboveskip=0pt,
    belowskip=0pt,
    frame=none,
    numbers=none,
    escapeinside={(*@}{@*)}
}

\newcommand{\promptrole}[1]{\textcolor{PromptRole}{\ttfamily\detokenize{#1}}}
\newcommand{\prompttag}[1]{\textcolor{PromptTag}{\ttfamily\detokenize{#1}}}

\definecolor{paceheader}{RGB}{240,240,240}
\definecolor{pacerow}{RGB}{232,239,250}
\newcommand{\inc}[1]{\textcolor{green!50!black}{\scriptsize~(+ #1)}}

\title{Where Does Staleness Accumulate? Pool Aware Effective Staleness Control for Asynchronous RL in LLM Post-Training}
\author[1,2]{Chenliang Li}
\author[1,\dagger]{Neiwen Ling}
\author[1]{Zijun Wei}
\author[3]{Alfredo Garcia}

\affiliation[1]{ByteDance}
\affiliation[2]{Texas A\&M University}
\affiliation[3]{University of Virginia}

\contribution[\dagger]{Corresponding author}

\date{\today}

\abstract{
Fully asynchronous reinforcement learning (RL) improves resource utilization in large language model post-training by overlapping rollout generation with policy optimization, but it also introduces policy lag as trajectories are generated and queued while the trainer continues to update. We study how this lag accumulates over a trajectory's lifetime and how it can be controlled without sacrificing the wall-clock benefits of asynchronous execution. We decompose trajectory staleness into Generation Staleness, accumulated before rollout completion, and Waiting Staleness, accumulated after a completed trajectory enters the pool. Motivated by this decomposition, we introduce PACE (Pool-Aware Control of Effective Staleness). PACE converts excess pool occupancy into an adaptive rejection budget and ranks completed trajectories using an effective-staleness score that combines Waiting Staleness with prefix-aware Generation Staleness. This avoids penalizing long or interrupted rollouts solely because they span multiple policy versions. In single-turn mathematical reasoning, PACE improves the six-benchmark average validation accuracy by 18.7\% over unfiltered asynchronous RL at the same wall-clock budget and matches synchronous RL performance with 47.1\% less GPU time. PACE also improves validation performance in multi-turn tool-integrated reasoning, outperforming both synchronous and unfiltered asynchronous RL. Further experiments with the mixture-of-experts model and an alternative RL algorithm support its applicability across model architectures and training algorithms.
}

\begin{document}
\maketitle

\section{Introduction}

Fully asynchronous~\citep{fu2025areal} reinforcement learning (RL) for large language models offers a direct system benefit: rollout workers can continue generating trajectories while the trainer updates the policy. This is particularly useful for long-horizon reasoning~\citep{gao2025beyond,feng2026longcli} and tool-using tasks~\citep{he2026search,zhou2024webarena,feng2026retool}, where autoregressive generation is expensive and completion times vary widely across prompts and tasks~\citep{zhou2025april}. In a synchronous pipeline, optimization waits for the slowest trajectories in each batch, leaving resources idle at different stages of the iteration. Fully asynchronous systems remove this batch barrier and let generation and training proceed at their own pace.

The cost of this decoupling is that training data can become stale as the policy continues to evolve. Existing work addresses the resulting mismatch between behavior and training policies through off-policy correction~\citep{degris2012off,hou2026sao,li2025stabilizing}. Recent methods further improve the use of stale trajectories by controlling importance-weighted updates. M2PO constrains the second moment of importance weights to suppress extreme updates while preserving useful training signals~\citep{zheng2025prosperity}. VCPO combines learning-rate adaptation based on effective sample size with a variance-reducing baseline to stabilize asynchronous training~\citep{huang2026stable}. These studies show that stale trajectories can remain useful for learning, provided that the resulting off-policy updates are properly controlled. The benefit of overlapping rollout generation and training therefore depends on both system utilization and learning stability.

Complementary to these algorithmic approaches, asynchronous RL systems constrain policy lag to keep training data fresh. AReaL~\citep{fu2025areal} limits how far rollout generation can run ahead of training and prioritizes older trajectories when forming training batches. DORA~\citep{hu2026dora} maintains a bounded window of policy versions and advances it only after all trajectories from the oldest version have been collected and forwarded to the trainer. These constraints introduce a tradeoff: a tighter lag limit keeps data closer in version to the training policy but can force rollout generation or training to wait, whereas a looser limit allows greater overlap at the cost of staler training data. This tradeoff motivates a complementary perspective: controlling the flow of completed trajectories through the training pool. Selectively filtering surplus completed trajectories can reduce avoidable waiting while preserving asynchronous execution. The challenge is to do so without starving the trainer or discarding useful training signals. This leads to our central question:

\begin{center}
\emph{Under a fixed allocation of rollout and trainer resources,
how can we filter completed trajectories to maintain stable policy improvement while maximizing system utilization?}
\end{center}

\begin{figure}[t]
    \centering
    \includegraphics[width=0.55\linewidth]{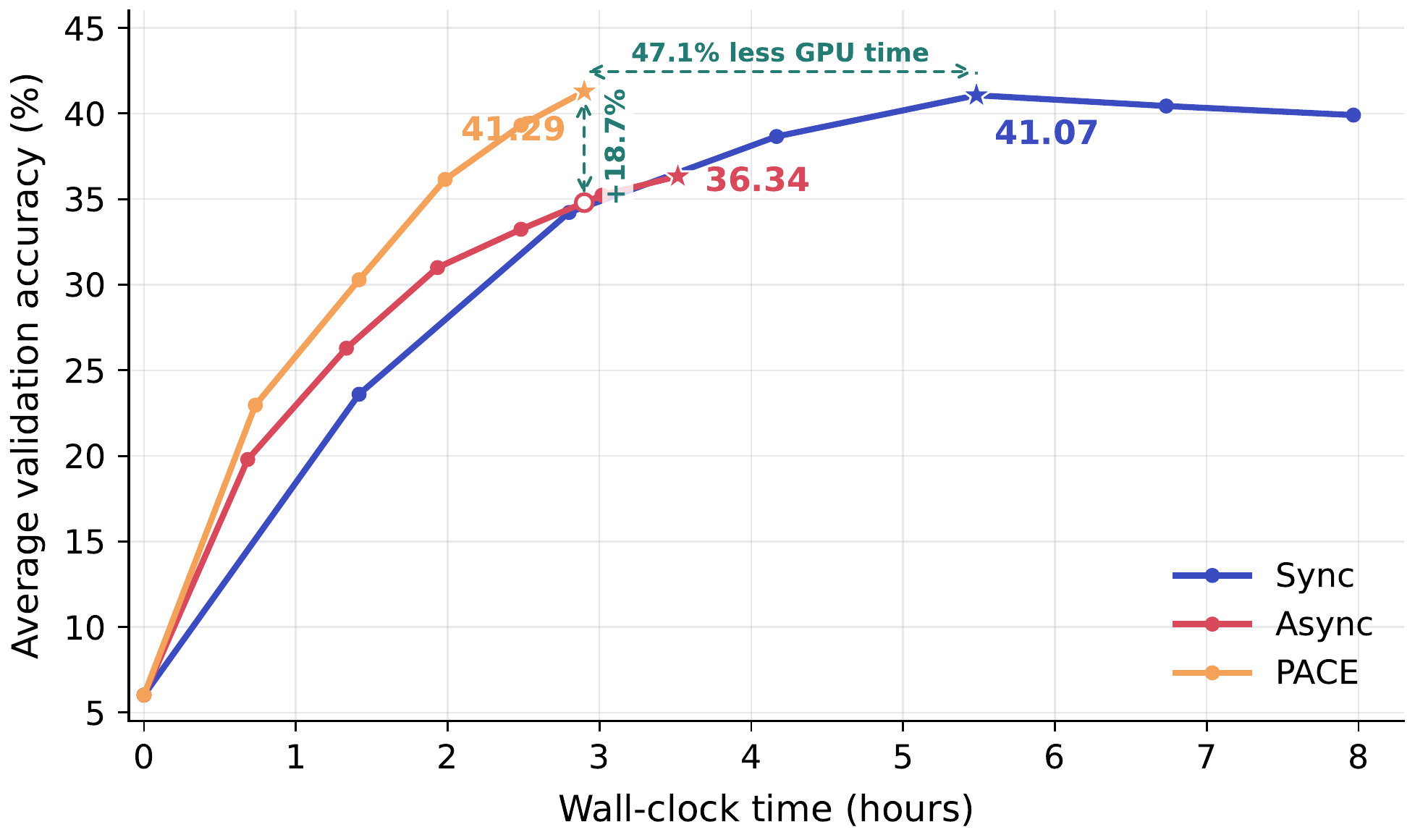}
    \caption{Average validation accuracy across six benchmarks as a function
    of cumulative logged training-step time. Markers correspond to checkpoints evaluated every 50 training steps, and stars indicate the best average accuracy achieved within the first 300 steps.   PACE reaches its peak in 2h54m, earlier than synchronous and asynchronous RL.}
    \label{fig:placeholder}

\end{figure}

Answering this question requires deciding \emph{how many} trajectories to reject and \emph{which ones}. We first decompose policy staleness into \emph{Generation Staleness}, accumulated before rollout completion, and \emph{Waiting Staleness}, accumulated in the pool afterward. In our unfiltered runs, persistent pool backlogs accompany elevated Waiting Staleness. We therefore use pool occupancy as feedback for the rejection budget, removing surplus when the pool is overfilled and relaxing rejection when ready data becomes scarce.

With the rejection budget determined by pool occupancy, we next consider which trajectories to remove. Ranking by raw staleness may disproportionately reject harder problems, since longer rollouts are more likely to accumulate generation lag~\citep{fu2025areal}. This can bias training toward shorter, easier examples. We introduce the Prefix Forgetting Score, which measures changes in prefix token log-probabilities, to refine the staleness signal.

We combine these components in \aname{} (Pool-Aware Control of Effective Staleness), as illustrated in Figure~\ref{fig:framework}. Pool occupancy sets an adaptive rejection budget, while an effective-staleness score ranks trajectories under that budget. The score retains Waiting Staleness in full and weights Generation Staleness by relative prefix drift. All filtering takes place at the completed-trajectory pool, where the trainer continues drawing trajectories until it has formed a full training batch. We evaluate \aname{} against synchronous and asynchronous RL on six mathematical reasoning benchmarks. As shown in Figure~\ref{fig:placeholder}, \aname{} reaches its peak average validation accuracy earlier than both baselines, at 2h54m of logged training-step time.

Our contributions are:
\begin{enumerate}
[leftmargin=*,itemsep=3pt,topsep=0pt,parsep=0pt,partopsep=0pt]
    \item We characterize how staleness accumulates in fully asynchronous RL. Our Stale-\(k\) experiments show that larger admitted policy lag slows early optimization, and we decompose trajectory staleness into Generation and Waiting Staleness.

    \item We identify persistent pool backlogs as a source of avoidable Waiting Staleness and establish pool occupancy as a natural signal for adaptive trajectory rejection.

    \item We introduce \aname{}, which combines a pool-aware
    rejection budget with prefix-aware trajectory ranking. Across six mathematical reasoning benchmarks, it improves learning over unfiltered asynchronous RL while retaining the time advantage of asynchronous execution.
\end{enumerate}

\begin{figure}[t]
    \centering
    \includegraphics[width=0.99\linewidth]{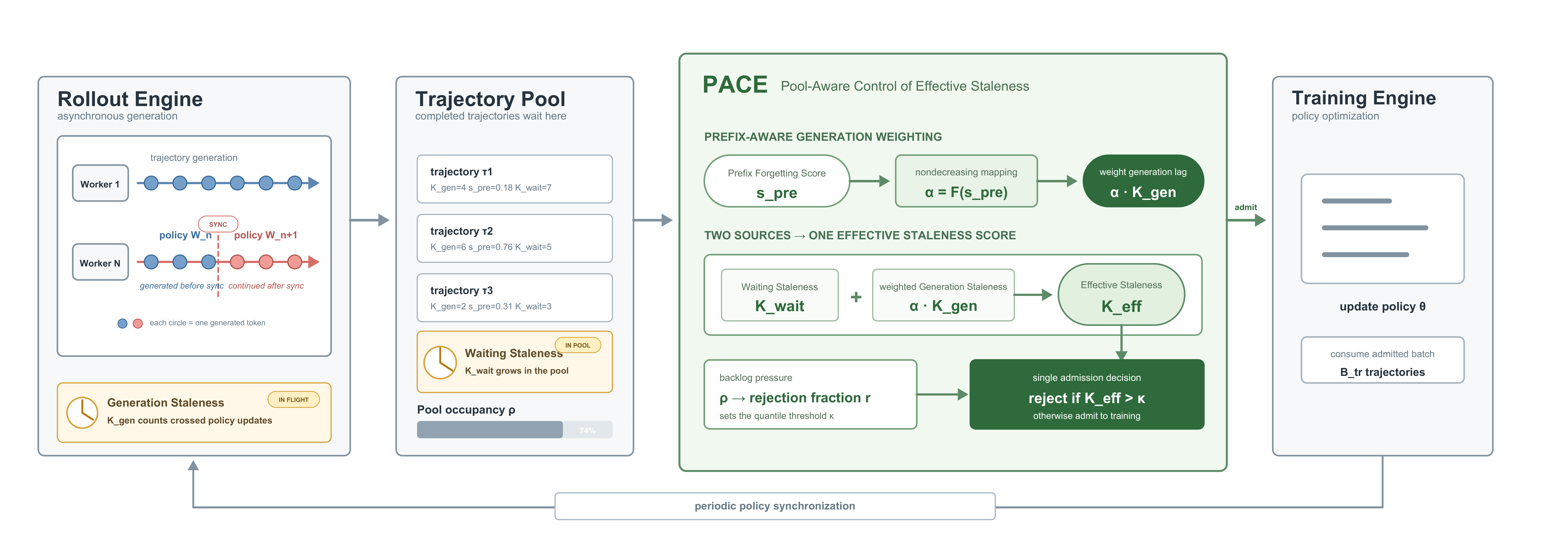}
    \caption{
    Overview of \aname{}.
    The {\color{gray}gray}-shaded modules depict the standard asynchronous RL pipeline, while the {\color{green}green}-shaded modules highlight our proposed components.
    \aname{} {uses an effective-staleness score} and pool occupancy to adaptively reject trajectories before training. }
    \label{fig:framework}
\end{figure}

\section{Where Does Staleness Accumulate?}
\label{sec:preliminaries}

As discussed above, asynchronous RL improves resource utilization but allows trajectories to become stale as training proceeds. Higher throughput therefore does not necessarily translate into faster policy improvement. In this section, we first define trajectory staleness and use controlled Stale-$k$ experiments to examine how increasing policy lag affects early optimization. We then decompose this lag into \emph{Generation Staleness}, accumulated before rollout completion, and \emph{Waiting Staleness}, accumulated in the pool afterward. This decomposition forms the basis of our method.

\textbf{Trajectory Lifecycle and Policy Lag.}
Let $\tau=(x,y)$ denote a complete trajectory, where $x\sim\mathcal{D}$ is a prompt and $y=(y_1,\ldots,y_T)$ is a
response of length $T=|y|$. An autoregressive language model
parameterized by $\theta$ defines a policy $\pi_\theta$, with
$\pi_\theta(y\mid x)=\prod_{t=1}^{T}\pi_\theta(y_t\mid x,y_{<t})$.
A reward model or verifier scores $(x,y)$, and the resulting
rewards are used to estimate token-level advantages $\hat A_t$.

In asynchronous RL, the trainer maintains a sequence of policy
versions $\{\theta_j\}_{j\geq 0}$, whose weights are periodically
synchronized to rollout workers. Let
$\mu_v:=\pi_{\theta_v}$ denote the behavior policy corresponding
to trainer version $v$. Generation proceeds concurrently with
training and may span multiple policy versions \citep{team2025kimi,fu2025areal,zhou2025april}. For each token $y_t$, the rollout worker records the policy version index $v_t$ used for generation and the corresponding log-probability $\ell_t^\mu=\log\mu_{v_t}(y_t\mid x,y_{<t})$. Let $c(\tau)$ denote the trainer version when trajectory $\tau$ completes generation. The completed trajectory is subsequently buffered until it is consumed at training step $j$. This step uses the current trainer policy $\pi_{\theta_j}$ and produces $\theta_{j+1}$. When trajectory $\tau$ is used for training at step $j$, its token-level policy lag is $k_t(\tau;j)=j-v_t$. Unless otherwise stated, we define the policy lag of a trajectory as the mean over its response tokens:
\begin{equation}
    \bar k(\tau;j)
    =
    \frac{1}{T}\sum_{t=1}^{T}(j-v_t).
    \label{eq:token-weighted-lag}
\end{equation}
This definition accommodates trajectories generated across
multiple policy versions and reduces to $j-v$ when all response
tokens are generated under a single version $v$.

\textbf{Controlled Stale-$k$ Analysis.}
To isolate the effect of policy staleness in asynchronous RL, we follow the controlled Stale-$k$ setting used in prior work~\citep{zheng2025prosperity,fu2025areal}, where only trajectories with a policy lag of at most $k$ are admitted for training. Synchronous RL serves as the on-policy baseline. For the asynchronous runs, we vary $k$ while keeping the remaining training configuration unchanged. As shown in Figure~\ref{fig:staleness-ablation}, larger values of $k$ slow early optimization, with the most pronounced delay at $k=8$. The AIME 2024 validation curves show a similar trend. These results suggest that limiting policy lag can accelerate policy improvement.

\begin{figure}[htbp]
    \centering
    \begin{minipage}[t]{0.45\linewidth}
        \centering
        \includegraphics[width=\linewidth]{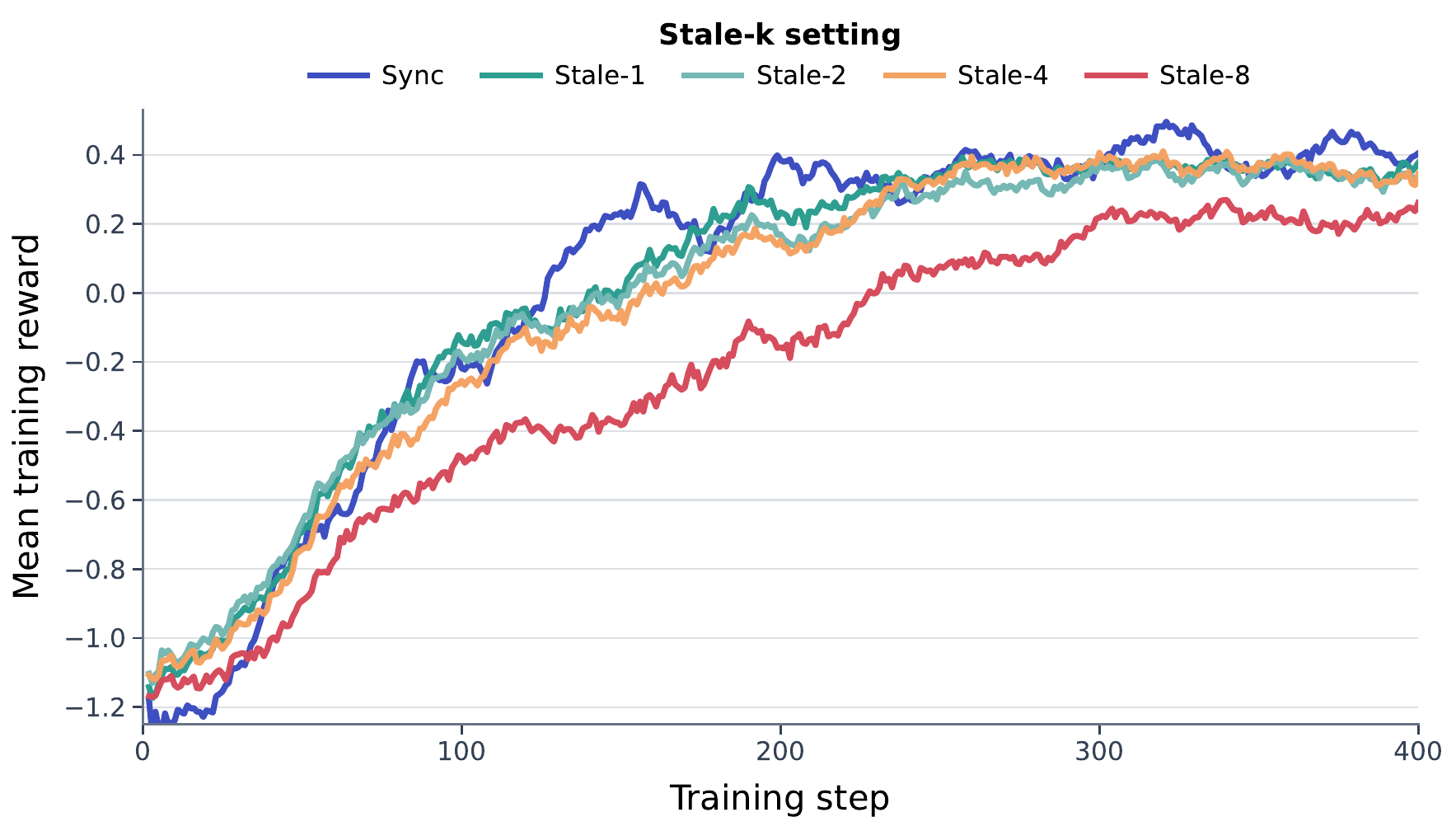}
        \small (a) DAPO-MATH-17K training reward.
    \end{minipage}\hfill
    \begin{minipage}[t]{0.45\linewidth}
        \centering
        \includegraphics[width=\linewidth]{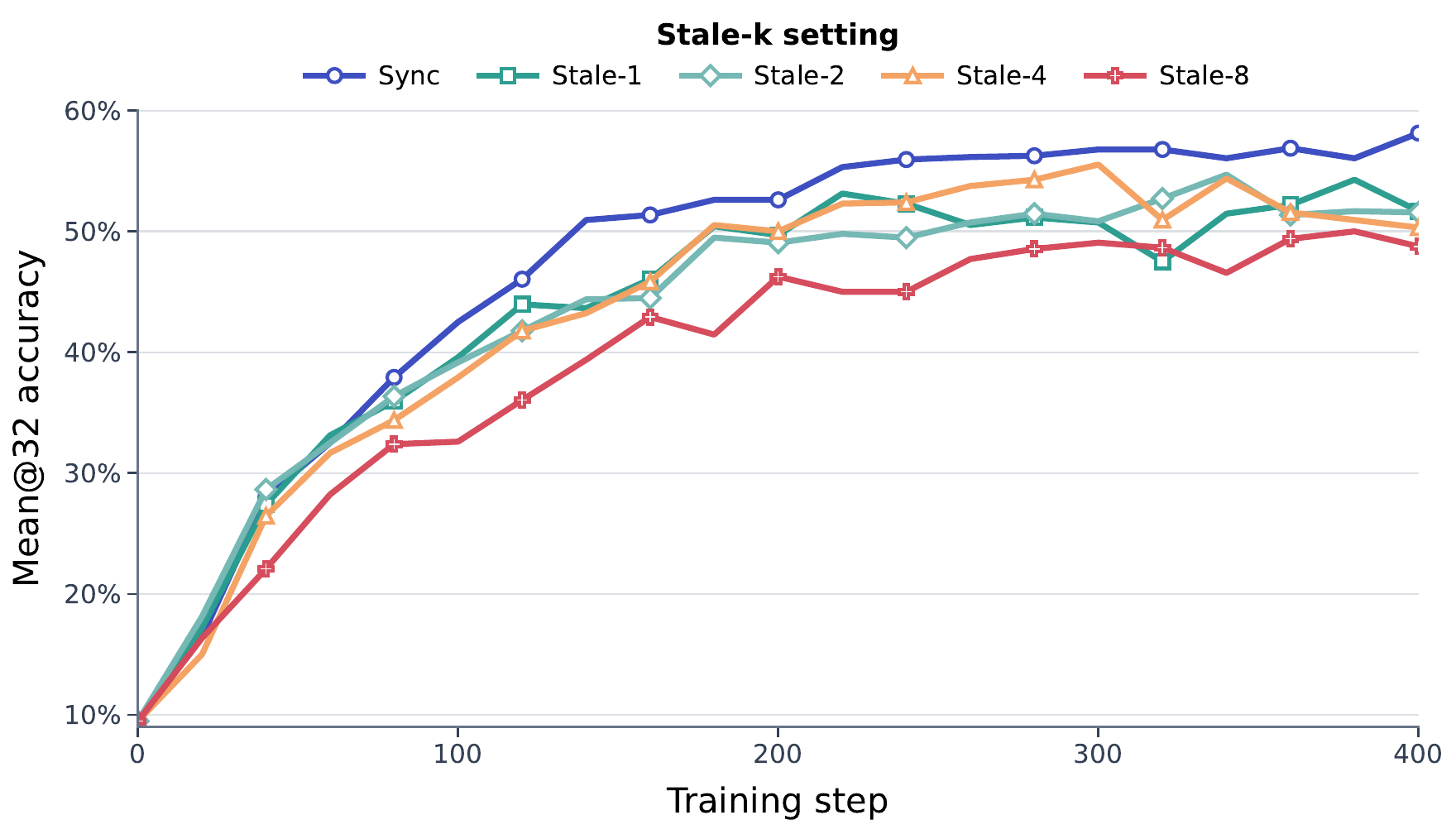}
        \small (b) AIME 2024 mean@32 accuracy.
    \end{minipage}
    \caption{
    Effect of the maximum admitted staleness $k$ on Qwen3-8B during the first
    400 training steps. Sync is the on-policy baseline. Stale-$k$ admits
    trajectories with policy lag at most $k$; for example, Stale-1 admits
    trajectories at most one policy update behind the trainer.
    }
    \label{fig:staleness-ablation}

\end{figure}

\textbf{Staleness Decomposition.}
{To understand where this lag arises, we examine two stages of a trajectory's lifetime. During generation, a rollout may span several trainer updates or policy-weight synchronization boundaries, producing \emph{Generation Staleness}. After completion, it may continue to age while waiting for training, producing \emph{Waiting Staleness}. Let $c(\tau)$ denote the trainer version when trajectory $\tau$ completes generation, and let $j$ denote the version at which it is consumed. Waiting includes any intermediate buffering between completion and consumption. The following lemma decomposes the policy lag defined in Equation~\ref{eq:token-weighted-lag} without presuming that either component dominates.}

\begin{lemma}[Staleness decomposition]
\label{lem:staleness-decomposition}
Consider a completed trajectory $\tau=(x,y)$ with $T$ response tokens, where token $t$ was sampled under behavior-policy version $v_t$. Suppose that $v_t \leq c(\tau) \leq j$ for all $t$. Then its staleness at training step $j$ decomposes as
\begin{equation}
    \bar k(\tau;j)
    =
    \underbrace{
        \frac{1}{T}\sum_{t=1}^{T}
        \left(c(\tau)-v_t\right)
    }_{\substack{
        K_{\mathrm{gen}}(\tau)
    }}
    +
    \underbrace{
        j-c(\tau)
    }_{\substack{
        K_{\mathrm{wait}}(\tau;j)
    }}.
    \label{eq:staleness-decomposition}
\end{equation}
\end{lemma}
The proof is provided in Appendix~\ref{app:proof-staleness-decomposition}.

\textit{Generation Staleness.}
$K_{\mathrm{gen}}(\tau)$ measures the policy-version lag accumulated before rollout completion. It can become large for long or interrupted trajectories simply because their generation spans multiple trainer updates. Consequently, a large $K_{\mathrm{gen}}$ does not necessarily imply that the current rollout policy assigns a substantially different likelihood to the realized prefix.

\textit{Waiting Staleness.}
$K_{\mathrm{wait}}(\tau;j)$ measures the additional lag accumulated after the trajectory has entered the completed-trajectory pool. Every token in the trajectory receives the same waiting component, and this component grows directly with pool residence time.

Taken together, the Stale-$k$ results and Lemma~\ref{lem:staleness-decomposition} show that controlling aggregate policy lag requires addressing two distinct sources.
Waiting Staleness arises after rollout completion and grows with pool residence time. The completed-trajectory backlog therefore provides a natural signal for controlling rejection pressure.
Generation Staleness accumulates while a rollout is in flight and depends on response length and synchronization timing. Version age alone is therefore an incomplete signal for ranking trajectories.
In the following section, we describe how we use these observations to design an efficient asynchronous RL algorithm.

\section{PACE}
\label{sec:method}

In fully asynchronous RL, rollout production and trainer consumption can proceed at different rates. When rollout supply is insufficient, the trainer must wait for new data. When production persistently exceeds consumption, surplus trajectories accumulate in the pool. This backlog adds Waiting Staleness without increasing the throughput of an already saturated trainer. We first examine how this backlog develops and then present \aname{} (Pool-Aware Control of Effective Staleness) to control it. Building on the decomposition in Section~\ref{sec:preliminaries}, we use pool occupancy to determine \emph{how much} data to reject and an effective-staleness score to determine \emph{which} completed trajectories to reject.

\subsection{Pool-Aware Control}

In this part, we describe how PACE determines the rejection budget and selects trajectories for training. We begin with the example in Figure~\ref{fig:staleness-decomposition} to motivate this design. We train Qwen3-8B on mathematical reasoning tasks using the same setup as the Stale-$8$ experiment and track Generation Staleness and Waiting Staleness over time.

As shown in Figure~\ref{fig:staleness-decomposition} (left), when rollout production outpaces trainer consumption, completed trajectories accumulate without improving the throughput of an already saturated trainer. Each trainer update adds one unit of Waiting Staleness to every trajectory remaining in the pool, so persistent backlog increases Waiting Staleness while Generation Staleness remains relatively stable. Under stable operating conditions, Little's law relates average pool occupancy to average residence time at a given departure rate~\citep{little2008little}. This motivates using occupancy above a target pool size to adapt the rejection budget online, removing surplus trajectories while retaining a buffer of ready data for training. Figure~\ref{fig:staleness-decomposition} (right) illustrates how controlling this backlog reduces Waiting Staleness.

\begin{figure}[!htbp]
\centering
\includegraphics[
width=\linewidth,
trim={0bp 0.4cm 0bp 0bp},
clip
]{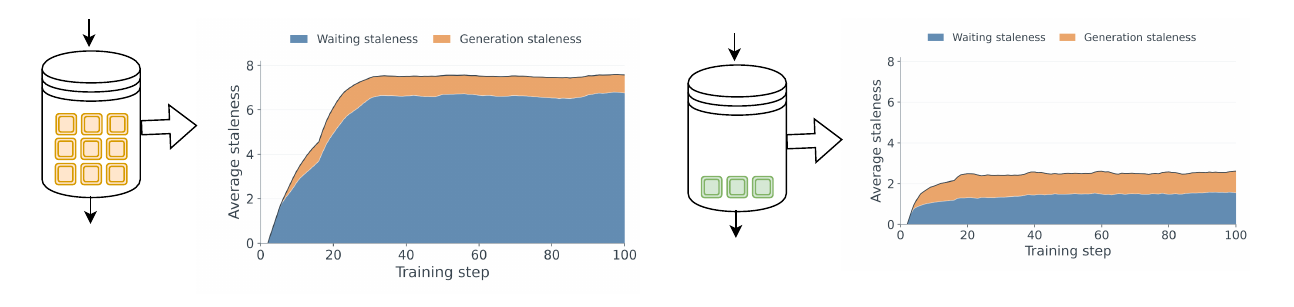}
\caption{Pool occupancy and staleness. A persistently backlogged pool (left) accumulates Waiting Staleness as completed trajectories wait for training. Removing surplus trajectories keeps the pool smaller and continually replenished (right), reducing this additional waiting.}
\label{fig:staleness-decomposition}

\end{figure}

To determine the rejection rate, we measure the fraction of trajectories above the target pool size. Let $N_j$ denote the number of completed trajectories in the pool at training step $j$, and let $N^{\mathrm{tar}}$ denote the desired buffer size of ready trajectories for training. When $N_j>N^{\mathrm{tar}}$, we treat the $N_j-N^{\mathrm{tar}}$ trajectories above this target as surplus. The instantaneous target rejection rate is then given by
\begin{equation}
    r_j
    =
    \operatorname{clip}
    \left(
        \frac{N_j-N^{\mathrm{tar}}}
             {\max(N_j,1)},
        0,
        1
    \right).
    \label{eq:occupancy-rejection-rate}
\end{equation}
Thus, $r_j=0$ when the pool is at or below its target occupancy, and a larger relative surplus leads to a higher target rejection rate. To avoid abrupt changes in the rejection budget as pool occupancy fluctuates, we smooth $r_j$ using an exponential moving average,
\begin{equation}
    \widehat r_j
    =
    \beta \widehat r_{j-1}
    +
    (1-\beta)r_j,
    \qquad
    \beta\in[0,1).
    \label{eq:smoothed-rejection-rate}
\end{equation}

The smoothed budget $\widehat r_j$ determines how much data to reject. We next specify which completed trajectories to discard using a score-based filtering rule.

Let $M(\tau;j)$ be a scalar score assigned to a completed trajectory $\tau$ at training step $j$,  with larger values indicating higher rejection priority. We will define its construction in Section~\ref{sec:design_metric}. To implement the rejection budget, we target the upper $\widehat r_j$ fraction of the score distribution. Using a sliding window $\mathcal{W}_j^M$ of recently observed scores, PACE estimates the corresponding cutoff as  $
    \kappa_j^M
    =
    \operatorname{Quantile}_{1-\widehat r_j}
    \left(\mathcal{W}_j^M\right).
    \label{eq:metric-quantile-threshold}
$
A completed trajectory is rejected when $M(\tau;j)>\kappa_j^M$, and the trainer continues drawing trajectories until a full training batch is formed. This rule preferentially discards trajectories with higher scores while targeting the rejection fraction set by the pool-aware budget. 

\subsection{Staleness-based Trajectory Scoring}
\label{sec:design_metric}

In this subsection, we design the trajectory score $M(\tau;j)$ used by the admission rule. We first introduce the Prefix Forgetting Score to measure policy drift on generated prefixes, then define a weighted staleness metric that accommodates different treatments of Waiting and Generation Staleness.

\textbf{Prefix Forgetting Score.}
We address Waiting Staleness by controlling pool backlog. For Generation Staleness, we instead examine how much the policy has changed on the generated prefix. Generation Staleness records the policy-version lag accumulated before rollout completion, but long or interrupted rollouts may span several trainer updates while the updated rollout policy still assigns similar likelihoods to their prefixes. Generation age alone therefore provides an incomplete basis for deciding which trajectories to reject. We use prefix rescoring to measure the change in likelihood on the realized prefix and calibrate the contribution of Generation Staleness.

Let $\xi=(x,y_{1:m})$ denote a realized prefix generated before a policy synchronization, and let $u$ denote the rollout-policy version under which the prefix is rescored. Using the behavior log-probabilities recorded during generation, we define the \emph{Prefix Forgetting Score} as the mean absolute log-ratio between the updated rollout policy and the behavior policies that generated the prefix:
\begin{equation}
    s_{\mathrm{pre}}(\xi;u)
    =
    \frac{1}{m}
    \sum_{t=1}^{m}
    \left|
        \log \pi_{\theta_u}(y_t\mid x,y_{<t})
        -
        \ell_t^{\mu}
    \right|.
    \label{eq:prefix-rescoring-score}
\end{equation}
Smaller scores indicate closer agreement with the recorded behavior likelihoods on the realized prefix. This provides a trajectory-specific signal of observed policy drift that complements generation age. We write $s_{\mathrm{pre}}(\tau)$ for the prefix score associated with a completed trajectory.

\textbf{Trajectory Score.} Using the decomposition in Section~\ref{sec:preliminaries}, we define the trajectory staleness score:
\begin{equation}
    M(\tau;j)
    =
    w_1 K_{\mathrm{wait}}(\tau;j)
    +
    w_2(\tau;j) K_{\mathrm{gen}}(\tau),
    \label{eq:weighted-staleness-metric}
\end{equation}
where $w_1\geq 0$ controls the waiting penalty and $w_2(\tau;j)\geq 0$ controls the generation penalty. The generation weight can depend on the trajectory's prefix score and recent score statistics. Different choices of these weights yield the metrics below, all of which use the same pool-aware rejection budget and quantile-based filtering rule.

\textbf{Effective Staleness.}
We use the Prefix Forgetting Score to adjust the generation penalty while retaining the full waiting penalty. Prefix rescoring during generation does not capture subsequent policy changes while the completed trajectory waits in the pool, so we retain Waiting Staleness as a conservative age penalty. Setting $w_1=1$ and $w_2(\tau;j)=F_j(s_{\mathrm{pre}}(\tau))$ gives
\begin{equation}
    K_{\mathrm{eff}}(\tau;j)
    =
    K_{\mathrm{wait}}(\tau;j)
    +
    F_j\!\left(s_{\mathrm{pre}}(\tau)\right)
    K_{\mathrm{gen}}(\tau),
    \label{eq:effective-staleness-general}
\end{equation}
where $F_j:\mathbb{R}_{\geq 0}\rightarrow[0,1]$ is nondecreasing. This weighting discounts generation age more strongly when the observed prefix drift is small.
To construct $F_j$, we measure each prefix score relative to recent observations. Using a sliding window $\mathcal{W}_j^{\mathrm{pre}}$ of recent prefix scores, we compute its empirical percentile rank as
\begin{equation}
    q_j(\tau)
    =
    \widehat{\operatorname{CDF}}_{\mathcal{W}_j^{\mathrm{pre}}}
    \!\left(s_{\mathrm{pre}}(\tau)\right)
    \in[0,1].
    \label{eq:prefix-score-quantile}
\end{equation}
A larger $q_j(\tau)$ indicates greater observed prefix drift relative to other scores in the window.
We then apply a smooth, monotone temperature-sharpening map~\citep{berthelot2019mixmatch}:
\begin{equation}
    F_j\!\left(s_{\mathrm{pre}}(\tau)\right)
    =
    \phi_\gamma\!\left(q_j(\tau)\right)
    :=
    \frac{q_j(\tau)^\gamma}
    {q_j(\tau)^\gamma+\left(1-q_j(\tau)\right)^\gamma},
    \qquad \gamma\geq 1.
    \label{eq:quantile-sharpening}
\end{equation}
For $\gamma>1$, this mapping pushes weights toward zero below the median and toward one above it, while preserving their ordering. Setting $\gamma=1$ recovers the linear map $\phi_1(q)=q$. The resulting weights reflect prefix drift relative to recent observations. We report a sensitivity analysis for $\gamma$ in Appendix~\ref{app:gamma_sensitive}.

\textbf{Other Variants.}
Setting $w_1=1$ and $w_2(\tau;j)=1$ yields \aname{} (Raw Staleness), which weights Waiting Staleness and Generation Staleness equally. Setting $w_1=0$ and $w_2(\tau;j)=F_j(s_{\mathrm{pre}}(\tau))$ yields the prefix-weighted generation-only variant \aname{} (Generation Staleness), which removes the Waiting Staleness term to examine its role in trajectory ranking. Our algorithm is summarized in Algorithm~\ref{alg:pace}. We then use two remarks to close this section.

\textbf{Remark 1.}
The prefix-aware weighting reduces the Generation Staleness penalty for long or interrupted rollouts when their observed prefix drift is small. This design is motivated by the potential bias of raw-staleness filtering toward easier examples, which we examine in Appendix~\ref{app:staleness-selection-bias}.

\textbf{Remark 2.}
\aname{} operates at trajectory admission and does not depend on a particular policy-optimization objective. It selects training data using pool occupancy and trajectory metadata, leaving the reward computation, advantage estimation, and policy loss unchanged. This separation allows integration with different single-turn and multi-turn asynchronous RL trainers, provided that they expose the required trajectory metadata and support prefix rescoring.

\section{Experiments}
\label{sec:experiments}

In this section, {we conduct a comprehensive evaluation} of our proposed \aname{} algorithm (Algorithm~\ref{alg:pace}) and compare its performance against the state-of-the-art baseline, M2PO {\citep{zheng2025prosperity}}. {Our experimental results demonstrate two key advantages} {of the proposed approach:} (1) pool-aware rejection curbs the accumulation of Waiting
Staleness by controlling the rollout backlog while maintaining training throughput; and (2) effective-staleness ranking distinguishes generation age from observed prefix-policy drift, enabling more selective trajectory filtering and supporting more stable training
and stronger validation performance.

\textbf{Experimental Setup.}
Following the experimental setup of DAPO~\citep{yu2026dapo}, we conduct experiments using Qwen3-8B and Qwen3-4B~\citep{yang2025qwen3} as base policies. We compare PACE against vanilla synchronous DAPO, asynchronous DAPO, and M2PO with DAPO as its base trainer. We train on DAPO-MATH-17K and evaluate all methods on six mathematical reasoning benchmarks: AIME 2023, AIME 2024, AIME 2025, BeyondAIME~\citep{seed2025seed1}, BrUMO 2025, and HMMT 2025~\citep{balunovic2026matharena}. {Table~\ref{tab:best_step300_results} reports mean@32 for all six benchmarks.} We also evaluate \aname{} in a multi-turn mathematical
reasoning setting based on ReTool. In this setting, the model interleaves natural language reasoning with code execution and incorporates execution feedback into subsequent reasoning. Full implementation details and hyperparameter settings are provided in the Appendix.

\textbf{Training Stability and Performance.}
Figure~\ref{fig:representative_benchmark_dynamics} compares the validation trajectories on AIME 2024 and BeyondAIME over the first 300 training steps.
All methods improve rapidly during early training, but their performance diverges as optimization progresses. On both benchmarks, \aname{} sustains improvements and
consistently achieves higher accuracy than unfiltered Async after the initial evaluation. Its performance approaches that of Sync on BeyondAIME and exceeds Sync on AIME 2024 during later training. These gains persist across multiple evaluations, showing that the advantage of \aname{} extends beyond a single best-performing checkpoint. The complete validation trajectories for the single-turn and multi-turn settings are provided in Appendix \ref{app:additional-results}.

Table~\ref{tab:best_step300_results} further compares the results at the evaluation with the highest six-benchmark average within 300 updates. For single-turn mathematical reasoning with Qwen3-8B, \aname{} achieves an average accuracy of 41.29\%, improving over Async by 4.95 percentage points and
matching the performance of Sync at 41.07\%. The improvement is larger in the multi-turn tool-integrated reasoning setting: \aname{} achieves 47.91\% average
accuracy, compared with 29.24\% for Async, 39.07\% for M2PO, and 45.06\% for Sync.
Together, these results demonstrate sustained gains over unfiltered asynchronous training and strong aggregate performance across both settings.

\begin{table}[!htbp]
    \centering
    \caption{
        Mean@32 validation accuracy (\%) at the evaluation with the highest six-benchmark average within 300 training steps. Each row reports all six benchmarks at the same evaluation; Avg.\ is their unweighted mean. Green values indicate percentage-point improvements over Async. Bold values indicate the best result among asynchronous methods within each setting. 
    }
    \label{tab:best_step300_results}

    \small
    \setlength{\tabcolsep}{4pt}
    \renewcommand{\arraystretch}{1.08}

    \resizebox{\linewidth}{!}{\begin{tabular}{lccccccc}
        \toprule
        Method
        & AIME23
        & AIME24
        & AIME25
        & BeyondAIME
        & BrUMO25
        & HMMT25
        & Avg. \\
        \midrule

        \rowcolor{paceheader}
        \multicolumn{8}{c}{
            \textit{Single-Turn Mathematical Reasoning
            (Base: Qwen3-8B)}
        } \\
        \midrule

        Sync
        & 45.00 & 55.83 & 41.46 & 27.34
        & 52.60 & 24.17 & 41.07 \\

        Async
        & 39.53 & 51.17 & 40.57 & 24.16
        & 40.94 & 21.65 & 36.34 \\

        M2PO
        & \textbf{48.13} & 56.88 &  40.42 & 27.31
        & 41.88 &  18.33 & 38.82 \\

        \aname{} (Generation Staleness)
        & 42.40 & 50.21 & 41.88 & 25.53
        & 38.49 & 18.41 & 36.15 \\

        \aname{} (Raw Staleness)
        & 42.50 & 53.91 & 40.42 & 24.49
        & 50.54 & \textbf{23.35} & 39.20 \\

        \rowcolor{pacerow}
        \aname{} 
        & 43.54\inc{4.01}
        & \textbf{57.66}\inc{6.49}
        & \textbf{43.44}\inc{2.87}
        & \textbf{27.81}\inc{3.65}
        & \textbf{52.40}\inc{11.46}
        & 22.92\inc{1.27}
        & \textbf{41.29}\inc{4.95} \\

        \midrule
        \rowcolor{paceheader}
        \multicolumn{8}{c}{
            \textit{Multi-Turn Tool-Integrated Reasoning
            (Base: Qwen3-4B-Instruct)}
        } \\
        \midrule

        Sync
        & 52.08 & 59.69 & 50.83 & 26.50
        & 48.13 & 33.13 & 45.06 \\

        Async
        & 39.79 & 37.50 & 30.83 & 16.38
        & 31.67 & 19.27 & 29.24 \\

        M2PO
        & 49.17 & 57.29 & 41.56 & 22.84
        & 35.94 & 27.60 & 39.07 \\

        {\aname{} (Generation Staleness)}
        & {53.33} & {56.35} & {47.40} & {27.19}
        & {49.17} & {31.67} & {44.18} \\

        \aname{} (Raw Staleness)
        & \textbf{58.33} & 57.81 & 48.44 & 27.19
        & 50.94 & 32.19 & 45.82 \\

        \rowcolor{pacerow}
        \aname{}
        & 57.81\inc{18.02}
        & \textbf{63.02}\inc{25.52}
        & \textbf{52.40}\inc{21.56}
        & \textbf{28.06}\inc{11.69}
        & \textbf{51.77}\inc{20.10}
        & \textbf{34.38}\inc{15.10}
        & \textbf{47.91}\inc{18.67} \\

        \bottomrule
    \end{tabular}}
\end{table}

\FloatBarrier

\begin{figure*}[!htbp]
    \centering
    \includegraphics[width=0.87\textwidth]
    {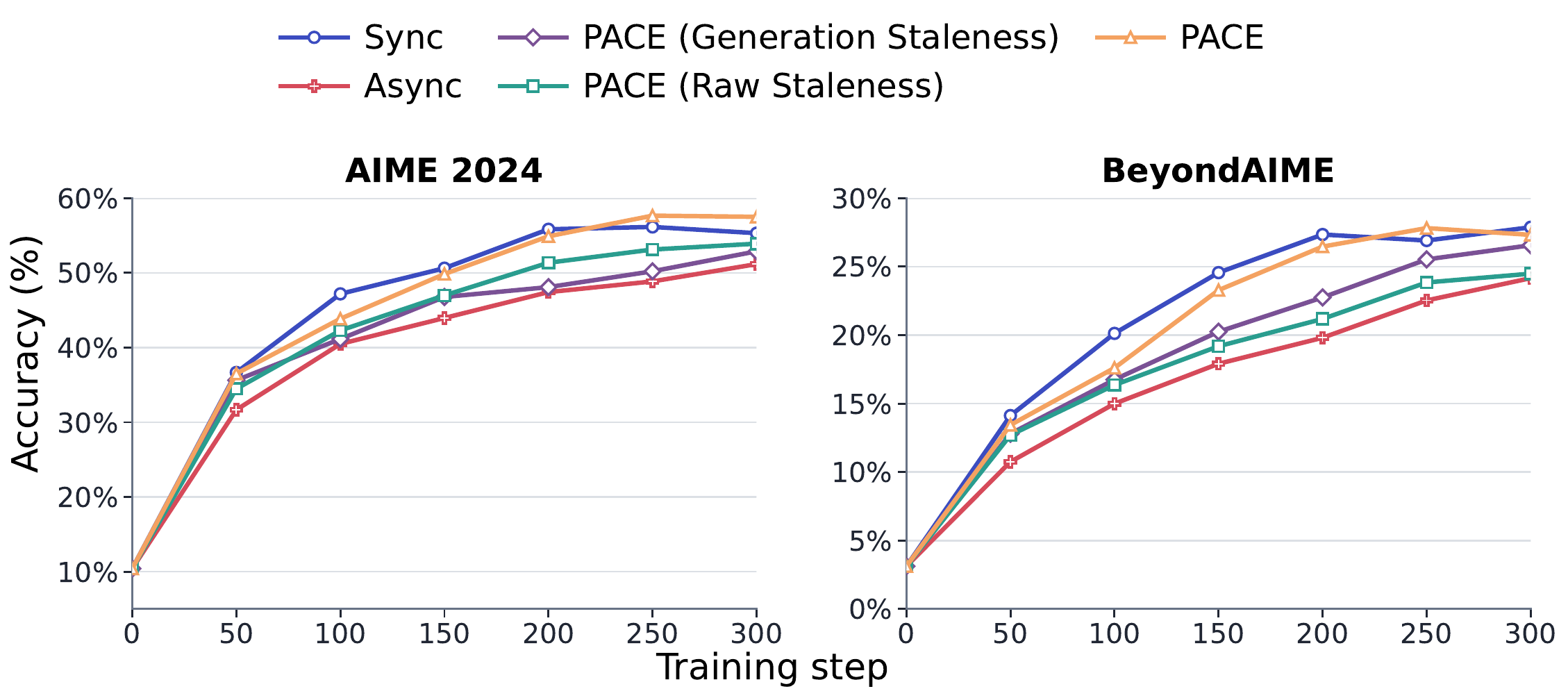}
    \caption{
    Representative mean@32 validation trajectories over the first
    300 training steps. Results for all six benchmarks are provided
    in Appendix Figure~\ref{fig:six_benchmark_dynamics}.}
    \label{fig:representative_benchmark_dynamics}

\end{figure*}

\textbf{Staleness Distribution Diagnostics.} We collect statistics over 50 training updates on the single-turn mathematical reasoning task. We record which candidate groups \aname{} retains or rejects in {Figure}~\ref{fig:pace_admission_distributions} and
compare the groups used for training under \aname{} and unfiltered Async in {Figure}~\ref{fig:async_pace_metric_distributions}.

Figure~\ref{fig:pace_admission_distributions} shows that groups with larger Waiting, Generation, and Effective Staleness tend to be rejected more often. As candidates are considered for training, \aname{} filters out high-staleness groups according to the pool-aware rejection budget. Removing surplus data helps keep the pool near its target size, allowing completed trajectories to enter training sooner and lowering the
Waiting Staleness of the data consumed by the trainer. Figure~\ref{fig:async_pace_metric_distributions} compares the groups entering training. With \aname{}, the Waiting and Effective Staleness distributions shift toward lower values, with fewer groups concentrated at large lags. This matches the goal of controlling pool accumulation so that completed trajectories spend less time waiting for training. Generation Staleness shows no comparable decrease. It accumulates during rollout generation and depends on rollout duration and policy-version progression. A larger generation lag does not necessarily indicate
greater prefix-policy mismatch. \aname{} uses the prefix score to adjust this penalty, allowing trajectories with large generation lag to remain eligible when their
observed prefix drift is low relative to recent scores.

\begin{figure*}[!t]
    \centering
    \begin{subfigure}[t]{0.24\linewidth}
        \centering
        \includegraphics[width=\linewidth]{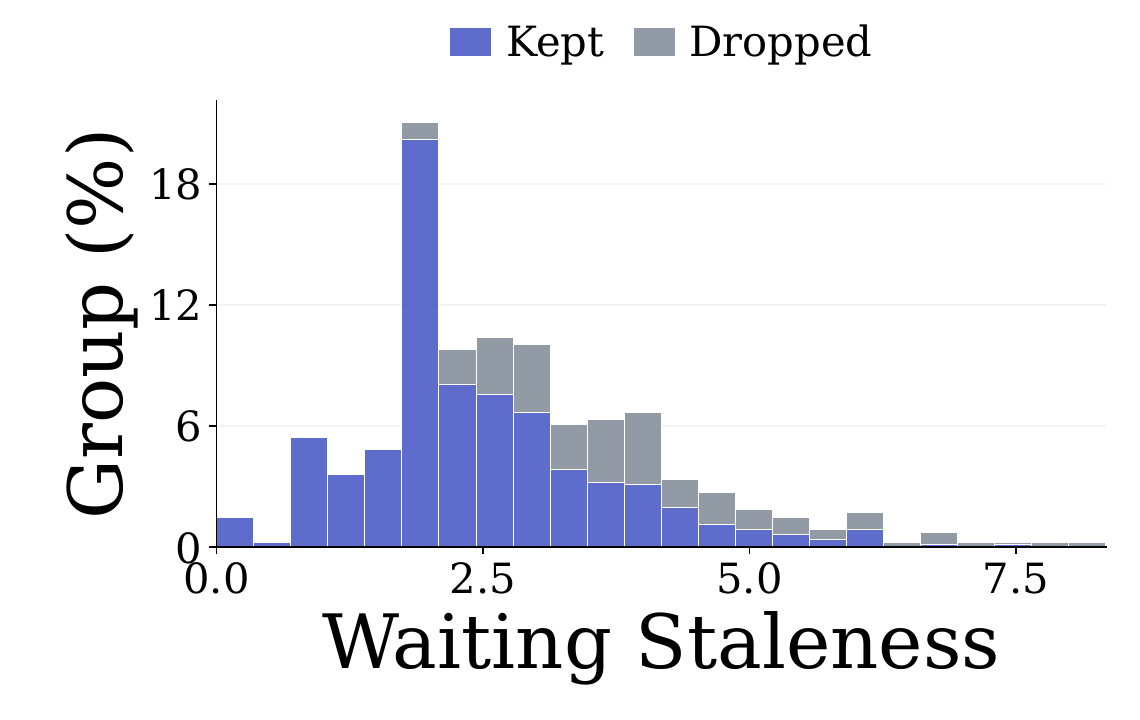}
        \caption{Waiting staleness.}
        \label{fig:pace_admission_waiting}
    \end{subfigure}\hfill
    \begin{subfigure}[t]{0.24\linewidth}
        \centering
        \includegraphics[width=\linewidth]{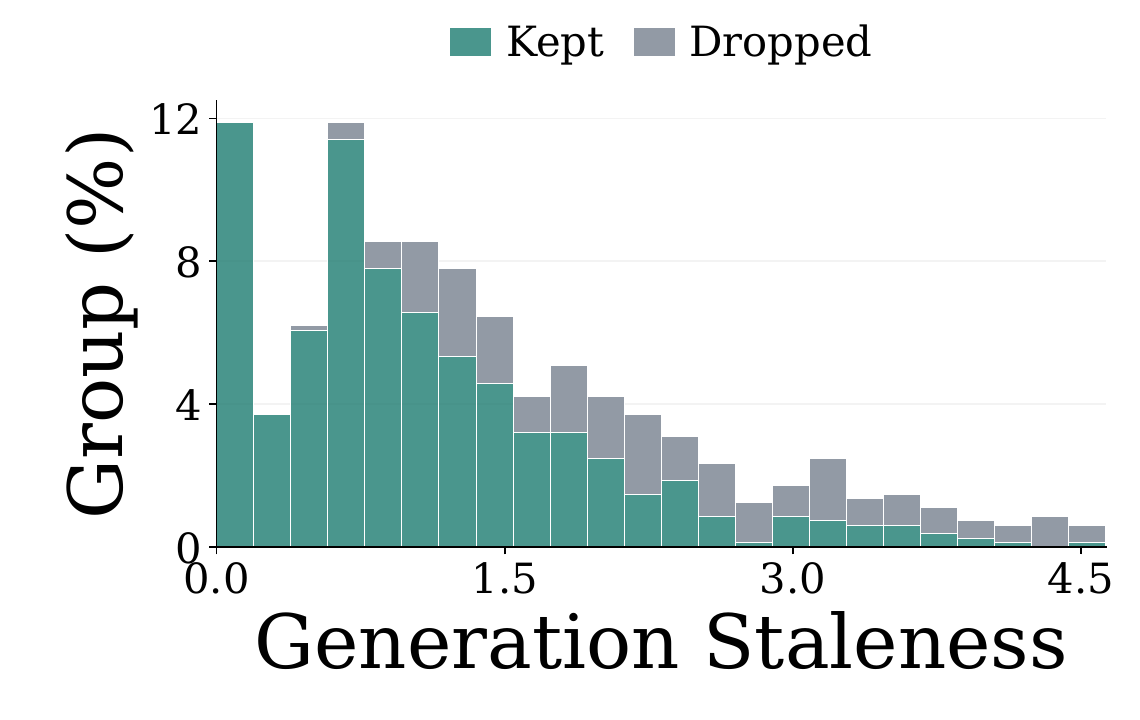}
        \captionsetup{justification=centering}\caption{Generation \mbox{Staleness.}}
        \label{fig:pace_admission_generation}
    \end{subfigure}\hfill
    \begin{subfigure}[t]{0.24\linewidth}
        \centering
        \includegraphics[width=\linewidth]{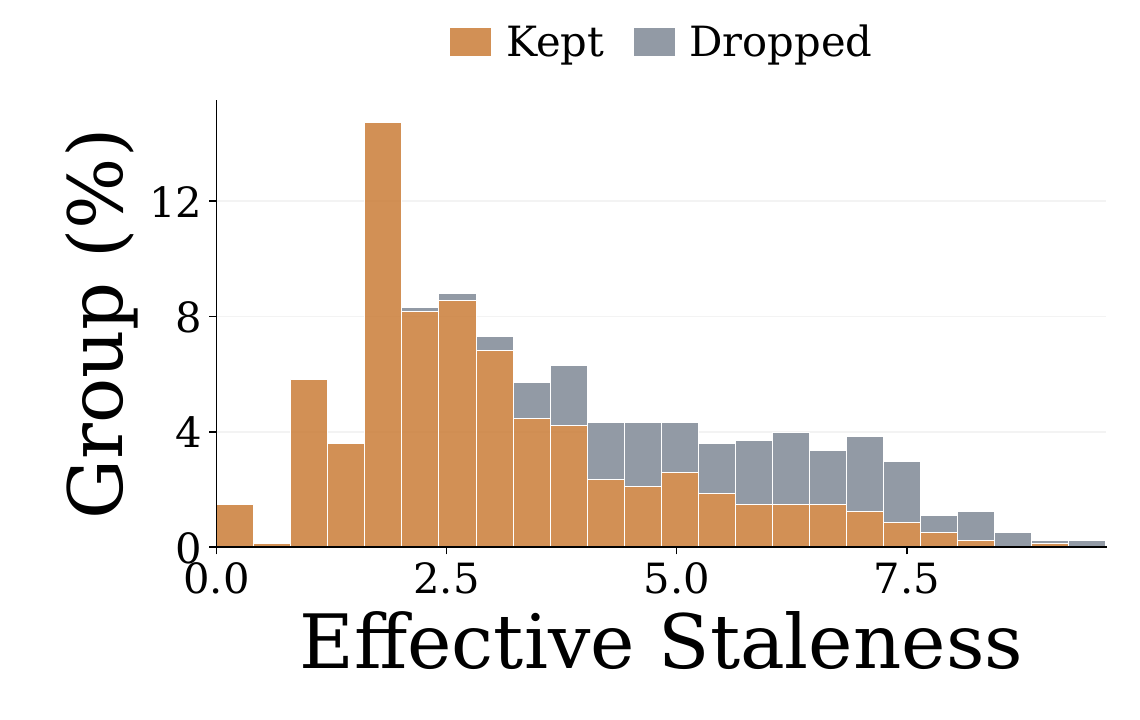}
        \caption{Effective staleness.}
        \label{fig:pace_admission_effective}
    \end{subfigure}\hfill
    \begin{subfigure}[t]{0.24\linewidth}
        \centering
        \includegraphics[width=\linewidth]{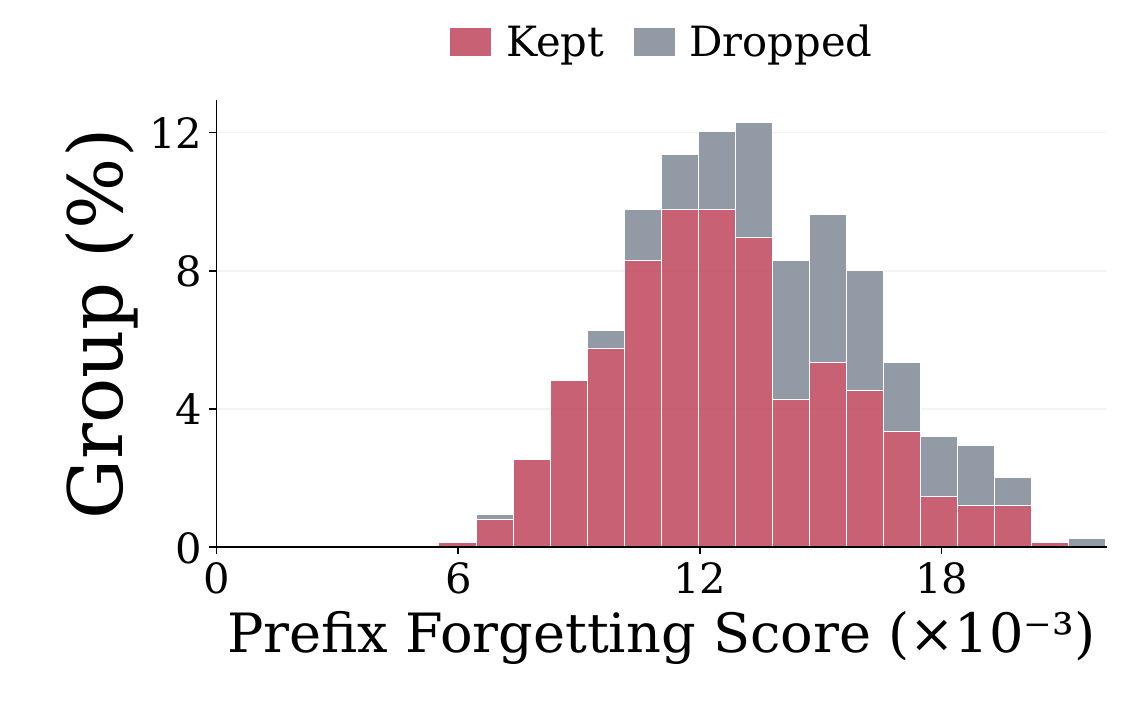}
        \caption{Prefix score.}
        \label{fig:pace_admission_prefix}
    \end{subfigure}
    \caption{\aname{} admission on math rollouts. Histograms show group-level metrics for all candidates encountered during the 50-update diagnostic: metric-specific colors denote retained groups, and gray denotes rejected groups. }
    \label{fig:pace_admission_distributions}

\end{figure*}

\begin{figure*}[!t]
    \centering
    \begin{subfigure}[t]{0.24\linewidth}
        \centering
        \includegraphics[width=\linewidth]{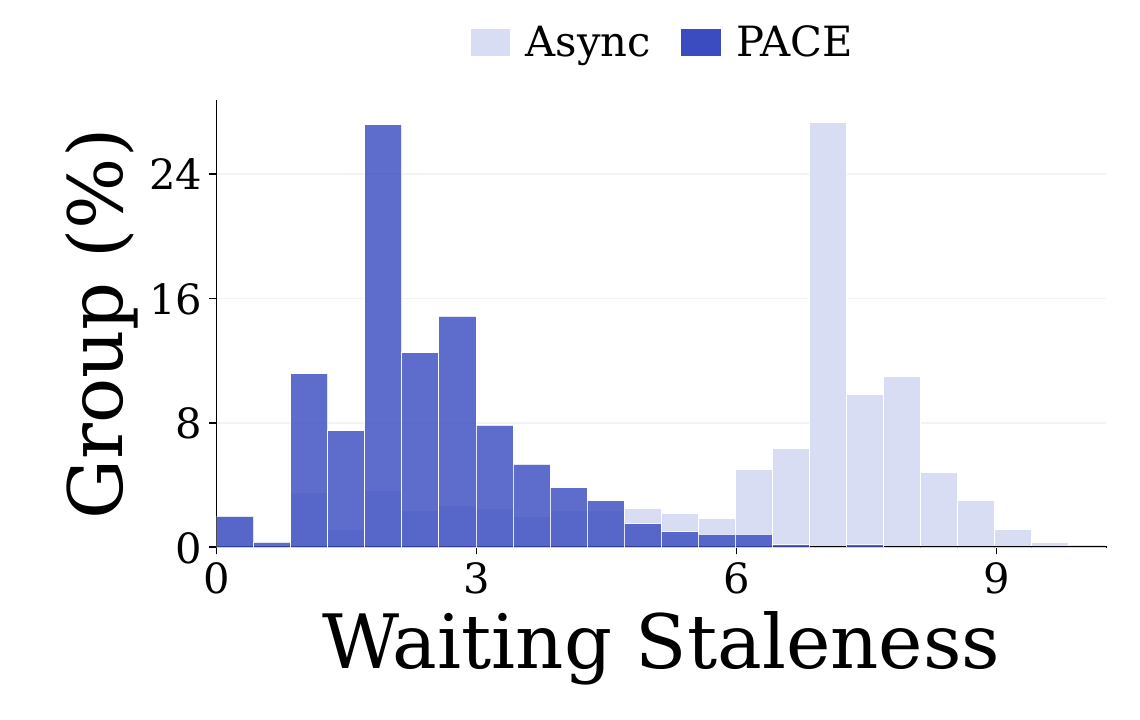}
        \caption{Waiting staleness.}
        \label{fig:async_pace_waiting}
    \end{subfigure}\hfill
    \begin{subfigure}[t]{0.24\linewidth}
        \centering
        \includegraphics[width=\linewidth]{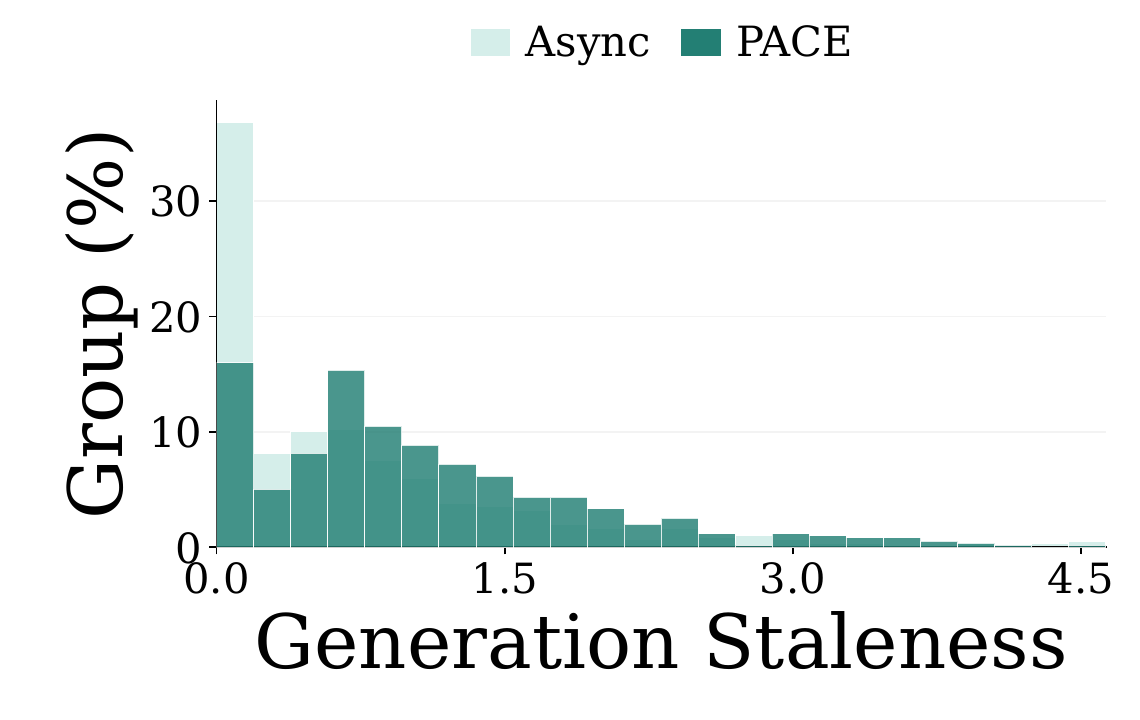}
        \captionsetup{justification=centering}\caption{Generation \mbox{Staleness.}}
        \label{fig:async_pace_generation}
    \end{subfigure}\hfill
    \begin{subfigure}[t]{0.24\linewidth}
        \centering
        \includegraphics[width=\linewidth]{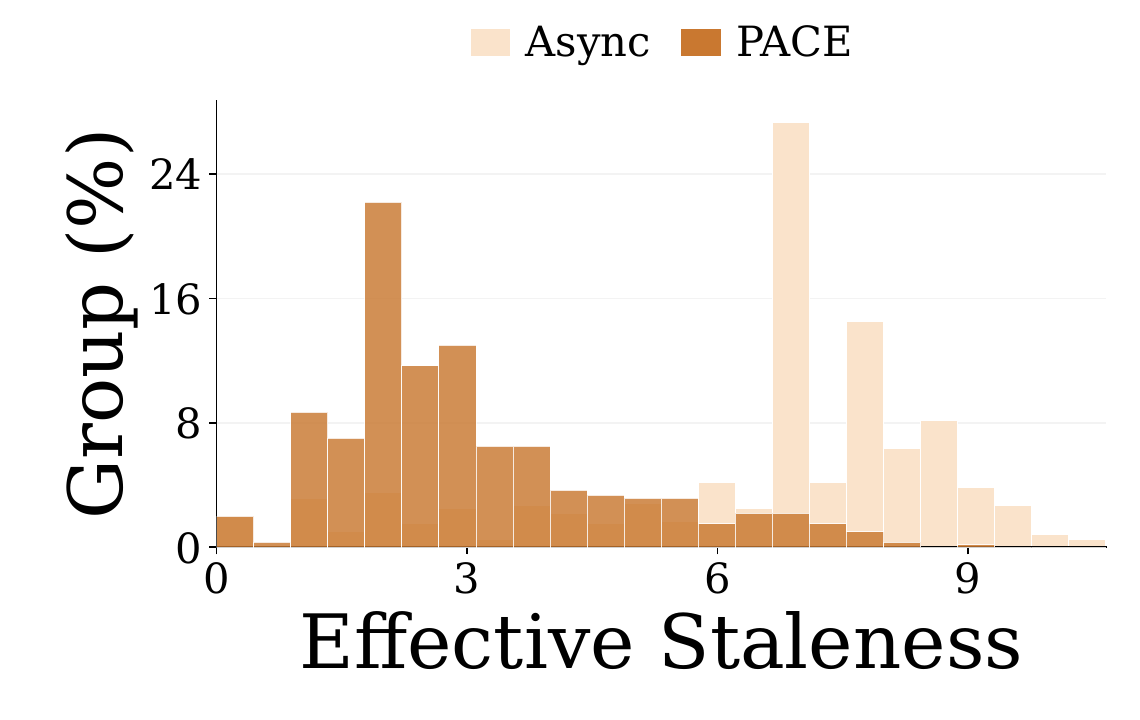}
        \caption{Effective staleness.}
        \label{fig:async_pace_effective}
    \end{subfigure}\hfill
    \begin{subfigure}[t]{0.24\linewidth}
        \centering
        \includegraphics[width=\linewidth]{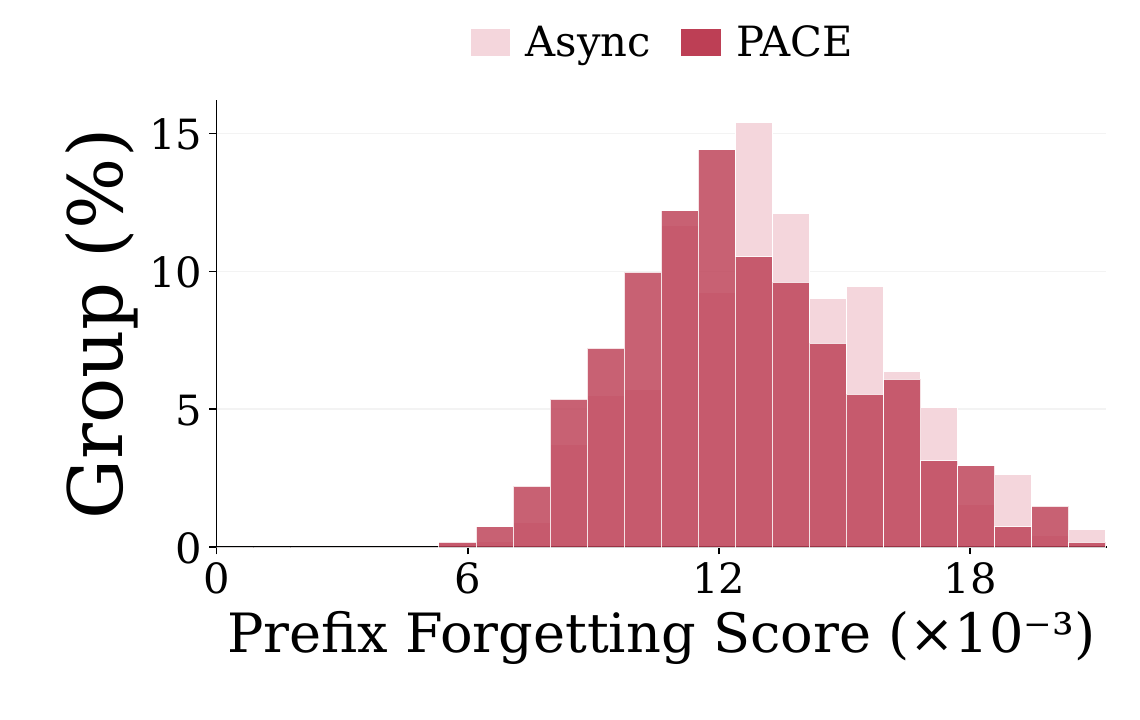}
        \caption{Prefix score.}
        \label{fig:async_pace_prefix}
    \end{subfigure}
    \caption{Distributions of groups entering training for \aname{} and unfiltered Async.
    Within each panel, lighter shading denotes Async; darker shading denotes groups retained by \aname{}.}
    \label{fig:async_pace_metric_distributions}

\end{figure*}

\textbf{Time Cost.} Figure~\ref{fig:retool_time_breakdown} compares the average wall-clock time per training step. \aname{} with raw and effective staleness takes 26.50 s and 28.22 s, respectively, compared with 26.19 s for vanilla Async and 97.01 s for Sync.
The small overhead of effective-staleness filtering comes from prefix rescoring, which partially overlaps with rollout-side waiting after synchronization (hatched region). Overall, \aname{} retains most of the wall-time benefit of asynchronous execution while improving training performance.

\begin{figure}
    \centering
    \includegraphics[
width=\linewidth,trim={0bp 0.3cm 0bp 0.5cm},
clip]{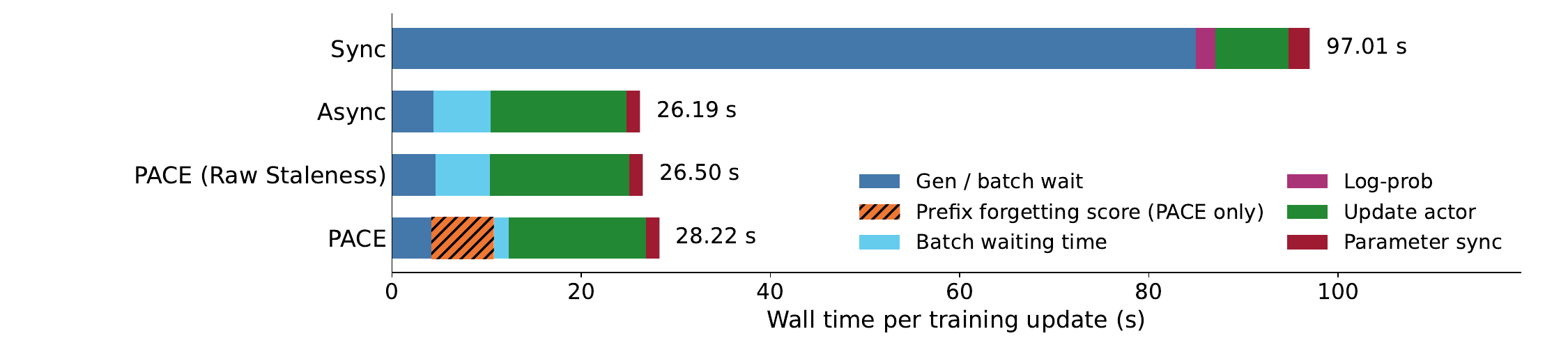}
    \caption{Wall-clock time breakdown per training update. All methods use six GPUs: Sync colocates rollout and training on all six, while asynchronous methods allocate four to rollout and two to training. Bars partition the trainer's elapsed update time into non-overlapping intervals that sum to the total. Hatched regions indicate prefix-scoring activity that overlaps rollout-related waits.}
    \label{fig:retool_time_breakdown}
\end{figure}

\textbf{Model Scaling.}
We further evaluate \aname{} on Qwen3-30B-A3B-Base~\citep{yang2025qwen3}, extending our study from dense 4B/8B backbones to a mixture-of-experts architecture with a larger total parameter count. As shown in Figure~\ref{fig:moe_scaling}, \aname{} achieves higher AIME 2024 accuracy than Async during training, while remaining competitive with Sync. Its per-step training time stays close to Async and substantially below the Sync baseline. 

\begin{figure}
    \centering
    \includegraphics[width=0.47\linewidth]{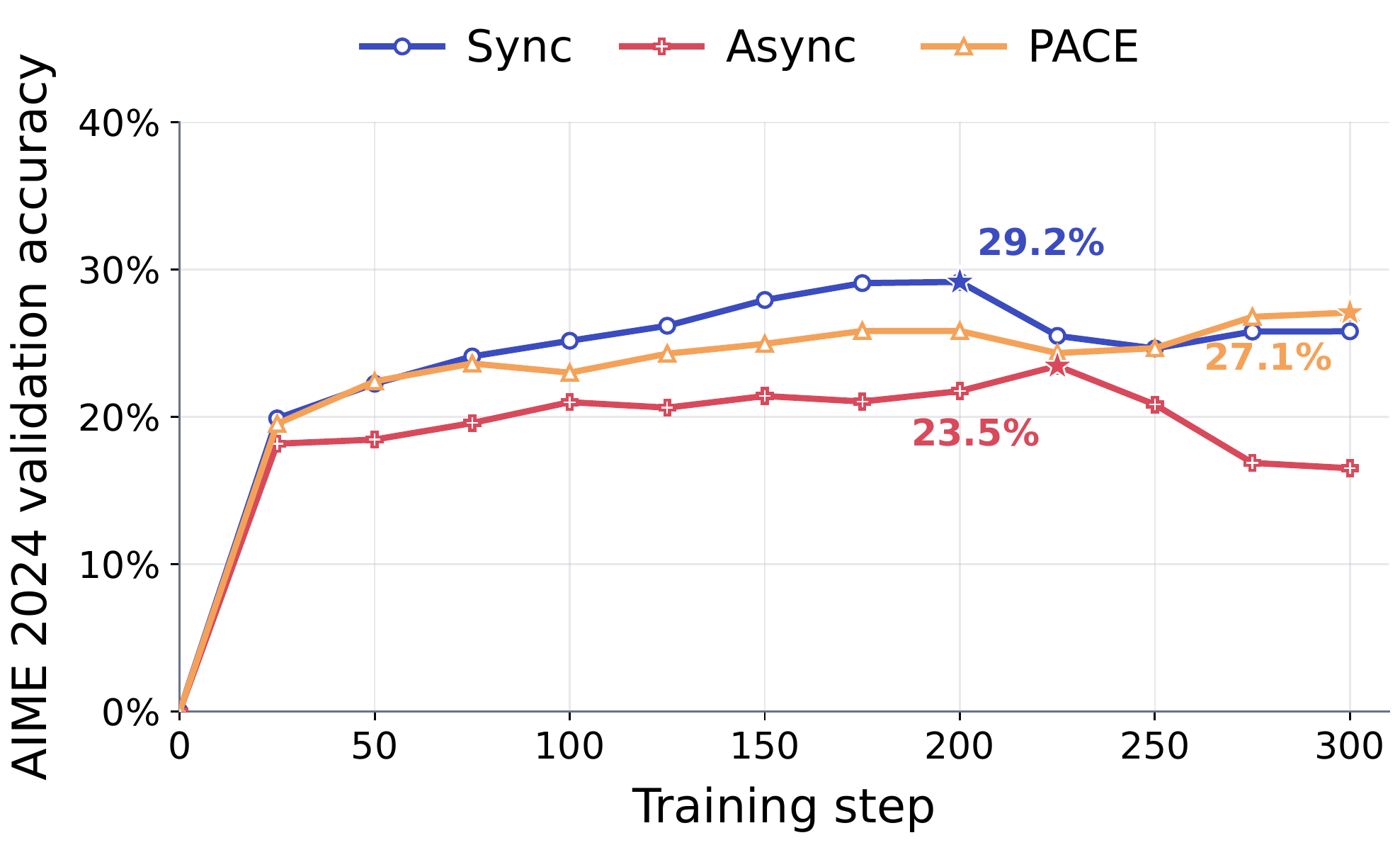}
    \includegraphics[width=0.47\linewidth]{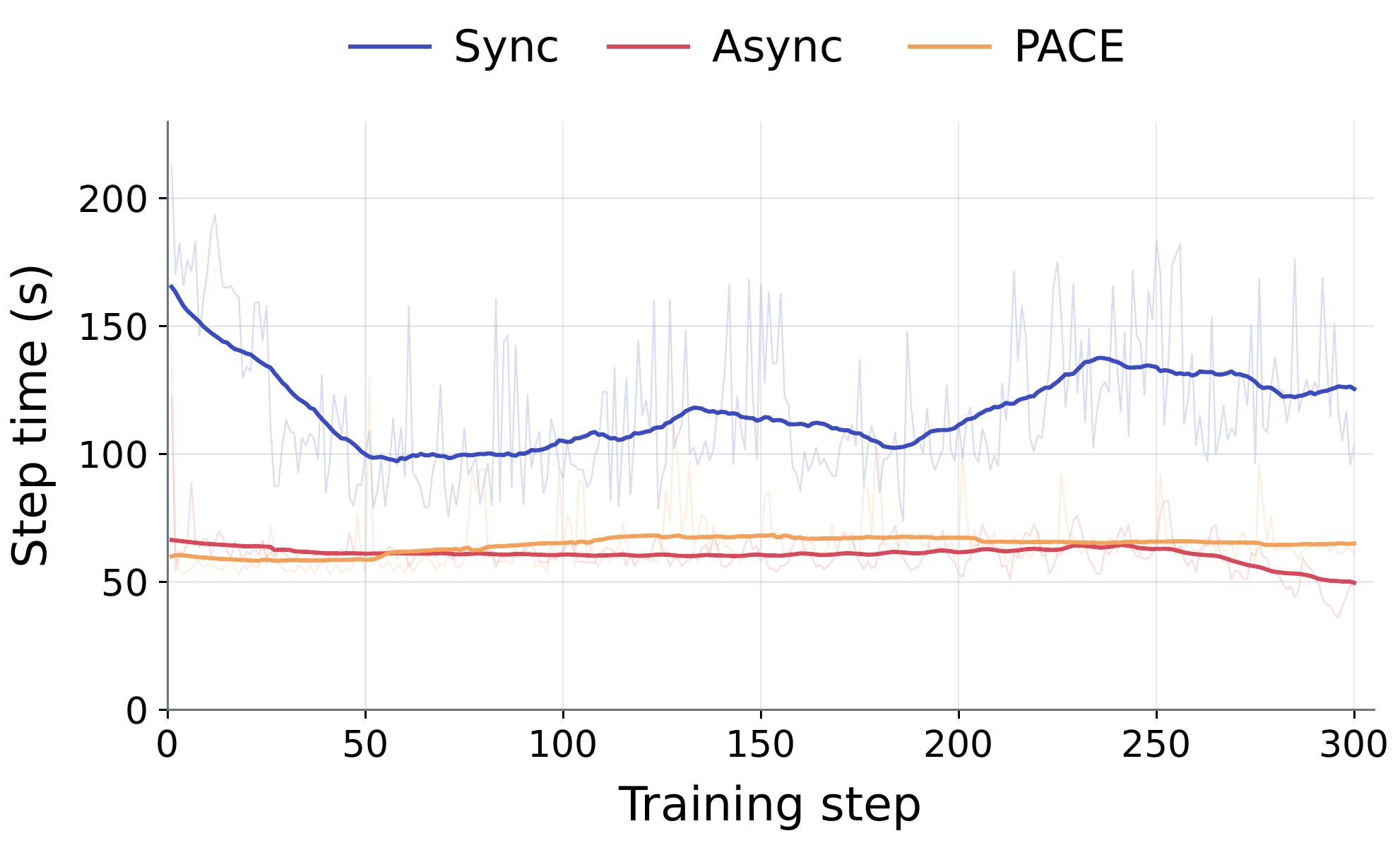}
    \caption{Scaling to Qwen3-30B-A3B-Base. Left: validation accuracy over 300 training iterations;  Right: per-iteration training time, excluding validation and checkpoint saving. Timing curves use a 50-iteration moving average, with faint traces showing raw times.}
    \label{fig:moe_scaling}
\end{figure}

\textbf{Ablation Studies.}
We examine staleness ranking, model scale, and the sharpening parameter.
\textbf{Staleness Scores.}
Table~1 compares Generation-Staleness, Raw-Staleness, and Effective-Staleness ranking. The effective-staleness score achieves the highest average validation accuracy, supporting the inclusion of Waiting Staleness and prefix-aware weighting of Generation Staleness.
\textbf{Model Scale.}
We compare Async, \aname{} (Raw Staleness), and \aname{} on the smaller Qwen3-4B backbone.
\textbf{Sharpening Parameter.}
We vary $\gamma\in\{1,2,4\}$ and observe similar validation and training-reward trajectories, suggesting limited sensitivity within the tested range.
Results for model scale and $\gamma$ are provided in Appendix~\ref{app:ablation-study}.

\FloatBarrier

\section{Related Work}
\subsection{Off-policy Optimization for LLMs}
{Off-policy updates can improve sample efficiency, but
heavy-tailed importance ratios can increase policy-gradient
variance~\citep{li2025st}. Existing methods address this through replay
design or control of importance weights. TOPR~\citep{roux2025tapered}
asymmetrically tapers REINFORCE~\citep{williams1992simple} importance
weights. ARPO~\citep{lu2025arpo} combines fresh and replayed samples to
reduce zero-advantage GRPO groups, while RePO~\citep{li2025repo} selects
past outputs using recency and reward. CISPO~\citep{chen2025minimax}
clips token-level importance weights, while GSPO~\citep{zheng2025group}
clips length-normalized sequence ratios.
M2PO~\citep{zheng2025prosperity} masks outlying tokens until the
{batch-level second moment of the importance weights} falls below a
threshold. In asynchronous training, off-policy data also arises without
explicit replay, as the trainer advances during rollout generation and
subsequent waiting.}

\subsection{Asynchronous RL Systems}
{Asynchronous RL systems overlap rollout generation with policy
optimization~\citep{hu2026dora,nemo-rl,primeintellect2025prime-rl,griggs2025skrylv01,cao2025skyrl,sheng2024hybridflow}.
APRIL~\citep{zhou2025april} mitigates stragglers by retaining unfinished
rollouts and resuming them in later iterations.
AReaL~\citep{fu2025areal} fully decouples generation from training and
pools trajectories whose segments may use different policy versions.
For long-horizon GRPO~\citep{shao2024deepseekmath},
SAO~\citep{hou2026sao} removes the group-level synchronization barrier
and corrects delayed rollouts using clipped token-level importance
ratios. PACE addresses trajectory admission in this setting: pool
occupancy determines the rejection budget, and effective staleness
determines the rejection order.}

\section{Conclusion}
\label{sec:discussion_conclusion}

In this work, we study where staleness accumulates in fully asynchronous RL and decompose policy lag into Generation and Waiting Staleness. Based on this distinction, we propose \aname{}, which uses pool occupancy to determine how much data to reject and effective staleness to determine which trajectories to reject. The score retains the full waiting penalty and uses prefix drift to adjust the generation penalty, avoiding rejection based on generation age alone. Our main experiments on single-turn and multi-turn mathematical reasoning show improved validation performance over unfiltered asynchronous training while preserving most of its wall-clock benefits. Further experiments with mixture-of-experts models and another RL algorithm support the applicability of \aname{} across model architectures and training algorithms. 

\newpage

\bibliography{main}
\bibliographystyle{plainnat}

\clearpage
\beginappendix
\raggedbottom
\setlength{\textfloatsep}{12pt plus 2pt minus 2pt}
\setlength{\floatsep}{10pt plus 2pt minus 2pt}
\setlength{\intextsep}{10pt plus 2pt minus 2pt}
\setlength{\dbltextfloatsep}{12pt plus 2pt minus 2pt}
\setlength{\dblfloatsep}{10pt plus 2pt minus 2pt}

\section{Proofs}
\label{app:proofs}

\subsection{Proof of Lemma~\ref{lem:staleness-decomposition}}
\label{app:proof-staleness-decomposition}

\begin{proof}
Starting from the definition of staleness in
Equation~\ref{eq:token-weighted-lag}, we add and subtract the completion
index $c(\tau)$ inside each token-level lag:
\begin{align*}
    \bar k(\tau;j)
    &=
    \frac{1}{T}\sum_{t=1}^{T}\left(j-v_t\right) \\
    &=
    \frac{1}{T}\sum_{t=1}^{T}
    \left[
        \left(c(\tau)-v_t\right)
        +
        \left(j-c(\tau)\right)
    \right] \\
    &=
    \frac{1}{T}\sum_{t=1}^{T}
    \left(c(\tau)-v_t\right)
    +
    \frac{1}{T}\sum_{t=1}^{T}
    \left(j-c(\tau)\right) \\
    &=
    \frac{1}{T}\sum_{t=1}^{T}
    \left(c(\tau)-v_t\right)
    +
    j-c(\tau).
\end{align*}
The first term is $K_{\mathrm{gen}}(\tau)$ and the second is
$K_{\mathrm{wait}}(\tau;j)$, which proves
Equation~\ref{eq:staleness-decomposition}. Moreover,
$v_t\leq c(\tau)\leq j$ implies that both components are nonnegative.
\end{proof}

\section{Additional Experimental Results}
\label{app:additional-results}

\begingroup
\subsection{Experimental Settings}
\label{app:experimental-settings}

{We describe the Qwen3-8B and ReTool runs in Table~\ref{tab:best_step300_results}, together with the additional Qwen3-4B and MoE experiments. Tables~\ref{tab:shared-experimental-settings} and \ref{tab:task-experimental-settings} separate shared choices from task-specific settings, including their evaluation intervals.}

\begin{table}[!htbp]
    \centering
    \caption{Shared settings for dense math, MoE, and ReTool experiments. {The clipping parameters apply to the standard policy objective; M2PO uses its own policy loss.}}
    \label{tab:shared-experimental-settings}
    \small
    \renewcommand{\arraystretch}{1.12}
    \begin{tabular}{ll}
        \toprule
        Setting & Value \\
        \midrule
        Training data & DAPO-MATH-17K \\
        Advantage estimator & GRPO with within-group standardization \\
        Peak learning rate & $10^{-6}$; constant after warmup \\
        Adam momentum coefficients & $(0.9, 0.999)$ \\
        Epochs per training batch & $1$ \\
        Loss aggregation & Token mean \\
        Lower / upper clipping widths & $0.20$ / $0.28$ \\
        Dual-clip coefficient & $10$ \\
        Gradient-norm clipping & $1.0$ \\
        KL penalty / entropy bonus & Disabled / disabled \\
        Training temperature / top-$p$ & $1.0$ / $1.0$ \\
        Maximum prompt length & $2{,}048$ tokens \\
        Generation engine & vLLM \\
        Dynamic accuracy-based group sampling & Disabled \\
        \bottomrule
    \end{tabular}
\end{table}

\paragraph{Shared optimization protocol.}
These experiments use group-relative advantages without a learned critic.
The standard objective combines asymmetric clipping with a token-mean
reduction; dynamic sampling that discards groups with identical rewards
is not enabled. Such accuracy-based sampling is distinct from
\aname{}'s staleness-based admission. In the dense math and current MoE
runs, the policy ratio uses recorded rollout log probabilities as its
behavior-policy reference. No additional rollout importance-weighting
or rejection-sampling correction is applied; the importance ratio in
the policy objective itself remains present.

\begin{table}[!htbp]
    \centering
    \caption{Task-specific settings. ReTool denotes the main GRPO-based runs; the separate REINFORCE++ comparison is described in Appendix~\ref{app:reinforcepp-ablation}.Validation periods are given in logged trainer steps; their conversion to optimizer updates is described below.}
    \label{tab:task-experimental-settings}
    \small
    \setlength{\tabcolsep}{4pt}
    \renewcommand{\arraystretch}{1.13}
    \begin{tabular}{lccc}
        \toprule
        Setting & Math (4B/8B) & MoE (30B-A3B) & ReTool (4B) \\
        \midrule
        Trainer backend & FSDP & Megatron & FSDP \\
        Optimizer & AdamW & Megatron Adam & AdamW \\
        Weight decay & $0.1$ & $0.1$ & $0.01$ \\
        Warmup (optimizer updates) & $10$ & $20$ & $0$ \\
        Response budget (tokens) & $8{,}192$ & $20{,}480$ & $8{,}192$ \\
        Responses per training prompt & $12$ & $16$ & $8$ \\
        Prompts per optimizer minibatch & $12$ & $24$ & $8$ \\
        Trajectories per optimizer minibatch & $144$ & $384$ & $64$ \\
        Rollout tensor parallelism & $1$ & $2$ & $1$ \\
        Validation responses per problem & 32 & 32 & $32$ \\
        Validation temperature / top-$p$ & $1.0/0.7$ & $1.0/0.7$ & $1.0/0.6$ \\
        Validation period & $50$ & $25$ & $25$ \\
        \bottomrule
    \end{tabular}
\end{table}

\paragraph{Single-turn math: Qwen3-4B.}
We initialize from Qwen3-4B and use four GPUs for rollout and two
for training. Each trainer update consumes 12 complete prompt groups,
followed by weight synchronization. The Sync baseline uses the same
disjoint resource layout with a zero-staleness supply setting;
Async and \aname{} use a supply threshold of eight and allow partial
rollouts across weight synchronizations. This threshold controls
outstanding work rather than imposing a strict upper bound on every
trajectory's measured staleness. The configured horizon is 400 updates, with no initial validation. Training uses the DAPO answer verifier, with signed-integer normalization and an overlong-response penalty applied over the final 4,096 tokens of
the response budget, with penalty factor one. Evaluation reports
unshaped answer accuracy on AIME 2023/2024/2025, BeyondAIME, BrUMO 2025, and HMMT 2025: 250 problems in total.

\paragraph{Single-turn math: Qwen3-8B.}
{We initialize from Qwen3-8B and retain the Qwen3-4B training settings, including the data, reward, optimizer, sequence-length limits, sampling parameters, and admission configuration. The main results use mean@32} The resource allocation spans two nodes, with 12 GPUs for rollout and four for training (12r4t).

\paragraph{Mixture-of-experts: Megatron 12r12t.}
We initialize from Qwen3-30B-A3B-Base. The asynchronous configuration allocates 12 GPUs to rollout and 12 to training; the matched synchronous configuration colocates generation and training on 24 GPUs. Training uses tensor and expert parallel degrees of two, with pipeline and context parallel degrees of one. Expert-routing replay and hybrid resource switching are disabled.
Each logged trainer iteration processes 96 prompt groups in four
minibatches of 24 groups. The asynchronous trainer synchronizes weights after these four optimizer updates. The continuous-rollout inventory is bounded at nine training batches (864 prompt groups, including unfinished work). The configured stopping point is 300 trainer iterations, {with AIME 2024 and AIME 2025 evaluated every 25 iterations. The plotted run uses mean@32.} Training uses the DAPO reward with a 4,096-token overlong buffer and penalty factor one.

\paragraph{Multi-turn reasoning: ReTool.}
We initialize from Qwen3-4B-Instruct-2507. The selected Sync run uses four GPUs and a data batch of 32 prompts, split into four optimizer minibatches of eight prompts. Async consumes one eight-prompt minibatch per logged step on a 2r2t allocation. The selected \aname{} variants and M2PO use 4r2t, with the same optimizer minibatch and eight responses per prompt. Accordingly, ReTool Sync evaluations are spaced by 100 optimizer updates, whereas asynchronous evaluations are spaced by 25; the displayed optimizer-step axis accounts for this difference.

The response budget spans the complete multi-turn interaction.
Assistant and tool-response turn limits are each set to eight, with one tool call processed at a time. The Hermes tool-calling template, SandboxFusion feedback, and turn-shaped training reward are specified in Appendix~\ref{app:retool-system-instruction}. Tool feedback remains in the context but is masked from the policy loss. Validation reports unshaped mean@32 accuracy on the same six benchmarks as the math evaluation. M2PO retains GRPO advantages and uses its core loss with budget $0.04$.

\paragraph{Admission and prefix-scoring parameters.}
The controllers above use a 512-group score window, a minimum of 32
observations, a five-update throughput window, and a maximum rejection budget of $0.9$. These are controller settings, not a fixed realized drop fraction. Full \aname{} in the dense math and ReTool experiments uses $\gamma=4$ and mean aggregation within groups. Prefix scores are collected at partial-rollout interruptions, requiring at least 32 eligible tokens; the scoring caps are 1,024 response tokens for math and 2,048 for ReTool. {\aname{} (Raw Staleness) omits prefix scoring and uses version lag directly.}

\FloatBarrier
\endgroup

\subsection{{Pseudocode}}
{Algorithm~\ref{alg:pace} summarizes the two-stage procedure in Section~\ref{sec:method}: rollout workers record policy-version and prefix-score metadata, while the trainer applies the pool-aware rejection budget to completed trajectories. The trainer retains data until a full batch is formed, then performs the original policy update and weight synchronization.}

\begin{algorithm}[!htbp]
\caption{\aname{}: Pool-Aware Control of Effective Staleness}
\label{alg:pace}
\begin{algorithmic}
\Require Completed-trajectory pool $\mathcal{Q}$, target occupancy $N^{\mathrm{tar}}$, trainer batch size $B_{\mathrm{tr}}$, smoothing coefficient $\beta$, and sharpening parameter $\gamma$.

\Statex \textbf{Rollout-side metadata collection}
\State Generate trajectories while recording token-level behavior versions and log probabilities.
\State When a prefix is rescored after a policy synchronization, attach its Prefix Forgetting Score to the trajectory; do not reject the rollout at this stage.
\State Insert each completed trajectory and its generation metadata into $\mathcal{Q}$.

\Statex \textbf{Trainer-side admission}
\For{training step $j=1,2,\ldots$}
    \State Compute the pool-aware rejection rate $r_j$ and its smoothed value $\widehat r_j$.
    \State {Compute the empirical percentile $q_j(\tau)$, apply $F_j=\phi_\gamma(q_j(\tau))$, and form $K_{\mathrm{eff}}=K_{\mathrm{wait}}+F_jK_{\mathrm{gen}}$.}
    \State Set $\kappa_j^M$ to the $(1-\widehat r_j)$-quantile of recent effective-staleness scores, with $M=K_{\mathrm{eff}}$.
    \State Draw from $\mathcal{Q}$, reject trajectories with $K_{\mathrm{eff}}(\tau;j)>\kappa_j^M$, and retain the rest until $|\mathcal{B}_j|=B_{\mathrm{tr}}$.
    \State Update the policy using $\mathcal{B}_j$ and periodically synchronize the latest weights to rollout workers.
\EndFor
\end{algorithmic}
\end{algorithm}

\subsection{Per-Benchmark Training Dynamics}
Figure~\ref{fig:six_benchmark_dynamics} reports the complete
validation trajectories over the first 300 training steps. The full
\aname{} method generally remains above unfiltered Async across the
six benchmarks, showing that the improvement in
Table~\ref{tab:best_step300_results} is not produced by a single
isolated checkpoint or benchmark. The advantage is particularly
visible on AIME 2024, BeyondAIME, and BrUMO 2025. HMMT 2025 exhibits
greater checkpoint-to-checkpoint variation, where {\aname{} (Raw Staleness)}
slightly outperforms the full method at the selected checkpoint.
Nevertheless, {\aname{}} provides the strongest
aggregate performance and the most consistent improvement over
unfiltered asynchronous RL.

\begin{figure*}[!htbp]
    \centering
    \includegraphics[width=\textwidth]
    {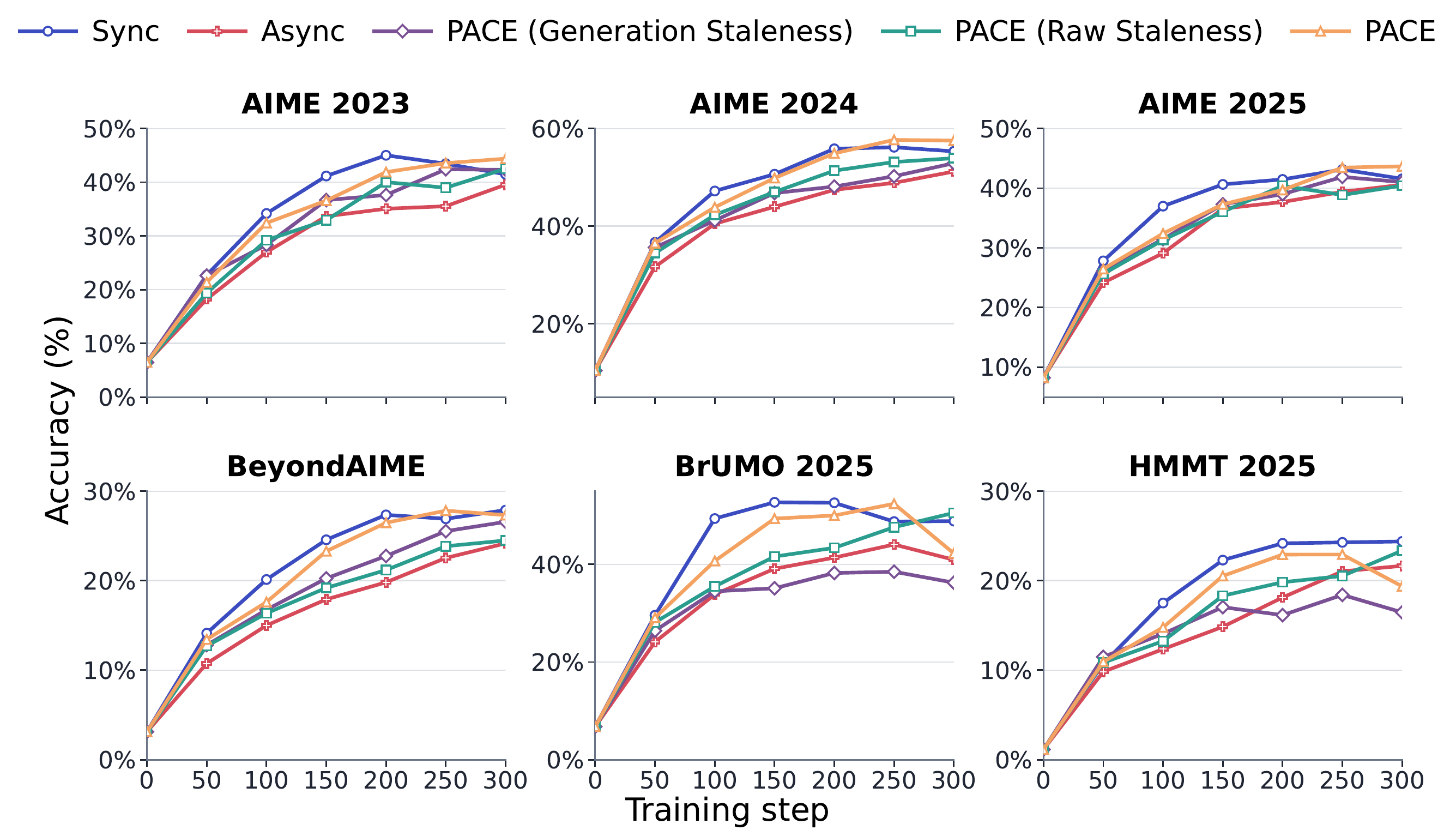}
    \caption{
    Mean@32 validation performance over the first 300 training steps
    on six math benchmarks. 
    Sync, Async, {\aname{} (Generation Staleness)}, {\aname{} (Raw Staleness)}, and {\aname{}}.
    }
    \label{fig:six_benchmark_dynamics}
\end{figure*}

\subsection{ReTool Validation Curves}
\label{app:retool-validation-curves}
Figure~\ref{fig:retool_validation_accuracy} provides the per-benchmark trajectories for the ReTool comparison in
Table~\ref{tab:best_step300_results}. The horizontal axis counts optimizer updates, with four updates per logged Sync step and one per logged Async or \aname{} step. 

\begin{figure*}[!htbp]
    \centering
    \includegraphics[width=\textwidth]
    {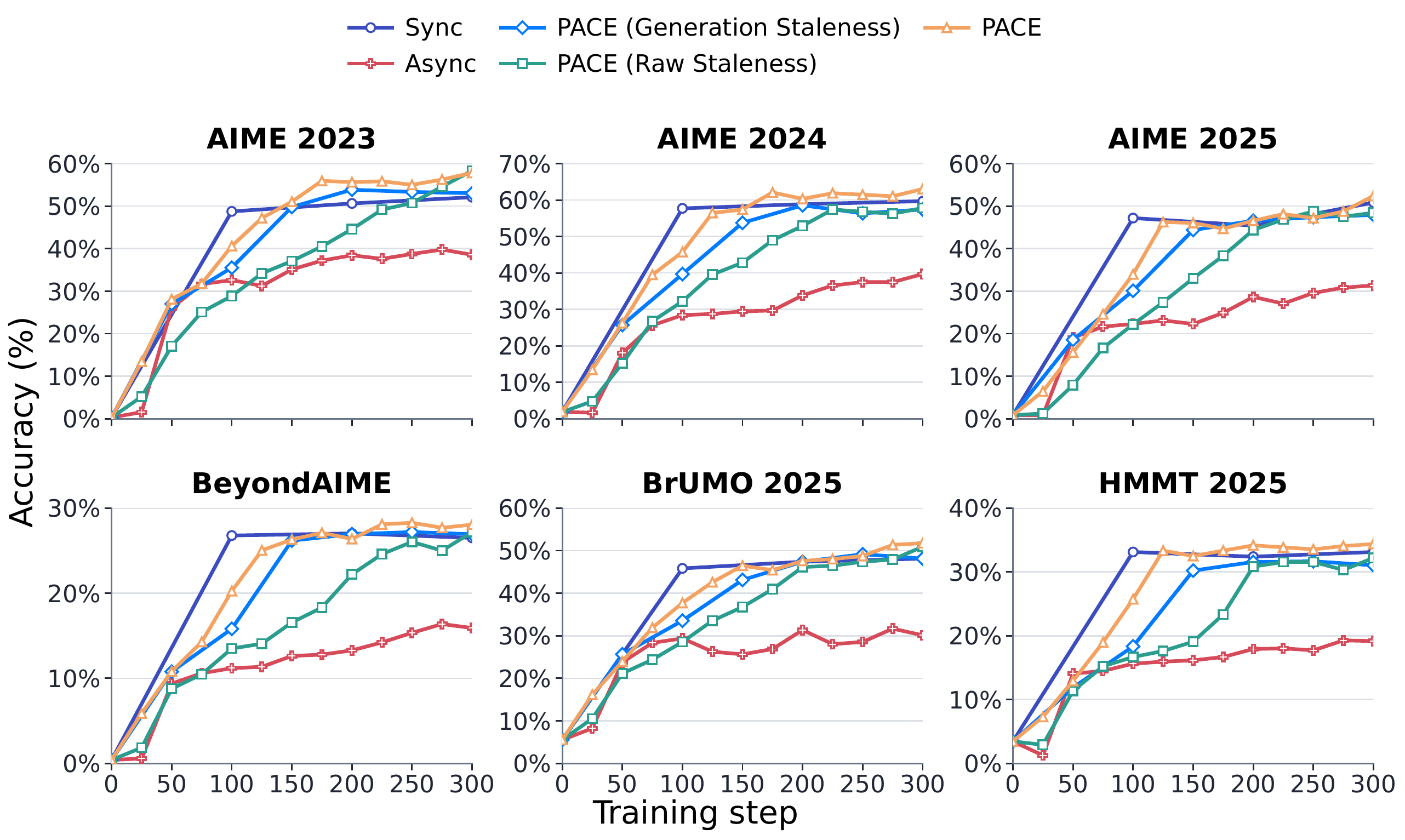}
    \caption{
    ReTool mean@32 validation accuracy over the first {300} optimizer updates on six benchmarks for Sync, Async, {\aname{} (Generation Staleness)}, \aname{} (Raw Staleness), and {\aname{}}.
    }
    \label{fig:retool_validation_accuracy}
\end{figure*}

\FloatBarrier

\section{Selection Bias under Raw-Staleness Filtering}
\label{app:staleness-selection-bias}

Figure~\ref{fig:policy-lag-selection-motivation} illustrates a potential selection bias introduced by raw-staleness filtering. In the left panel, \aname{} (Raw Staleness) achieves higher late-stage training reward than Sync under otherwise identical training settings. In the separate fixed-policy ReTool probe shown in the right panel, lower-accuracy prompt groups tend to have larger raw staleness. Since raw-staleness ranking assigns higher rejection priority to larger policy-version gaps, it can preferentially remove harder examples. Together, these observations suggest that the higher training reward may partly reflect a shift toward easier training data. Prior work also suggests that persistently favoring easier examples can limit further policy improvement~\citep{parashar2026curriculum,bae2026online}.

These observations motivate giving greater tolerance to Generation Staleness. Harder problems may require longer or interrupted rollouts and accumulate more generation lag while their prefixes remain compatible with the updated rollout policy. Applying the full generation-age penalty can therefore reject difficult but potentially informative trajectories. \aname{} discounts this penalty when the Prefix Forgetting Score indicates relatively small policy drift, while retaining the full Waiting Staleness penalty to account for policy updates after rollout completion.

\begin{figure*}[!htbp]
    \centering
    \begin{minipage}[t]{0.485\textwidth}

        \centering
        \includegraphics[width=\linewidth]{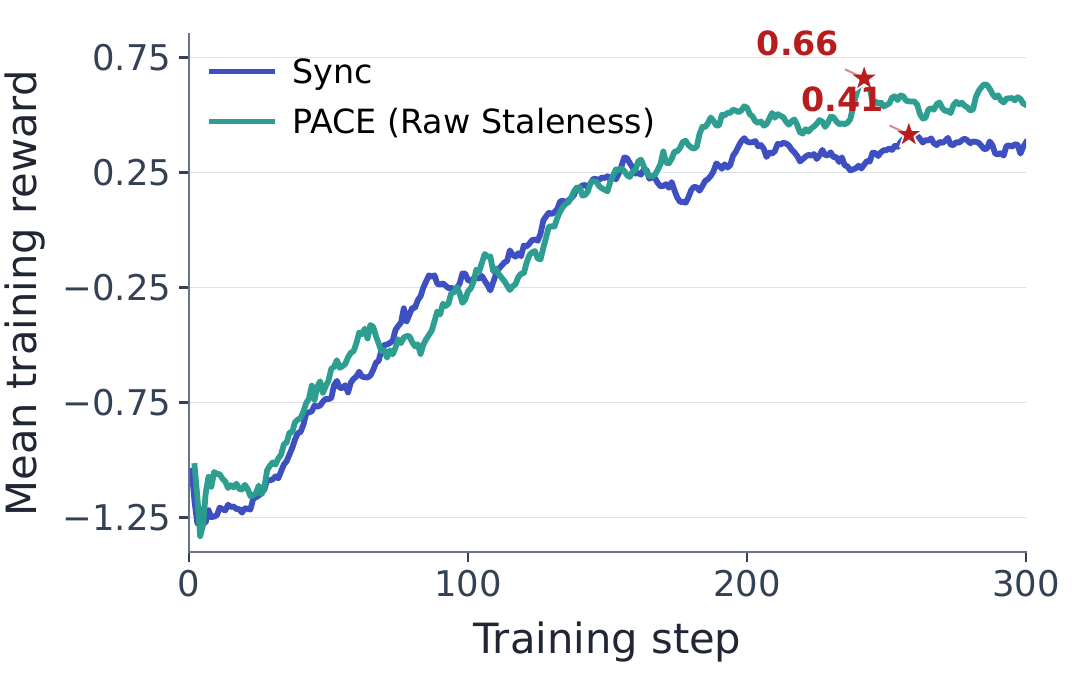}
        \par\smallskip
        {\small (a) Training reward.}
    \end{minipage}
    \hfill
    \begin{minipage}[t]{0.485\textwidth}

        \centering
        \includegraphics[width=\linewidth]{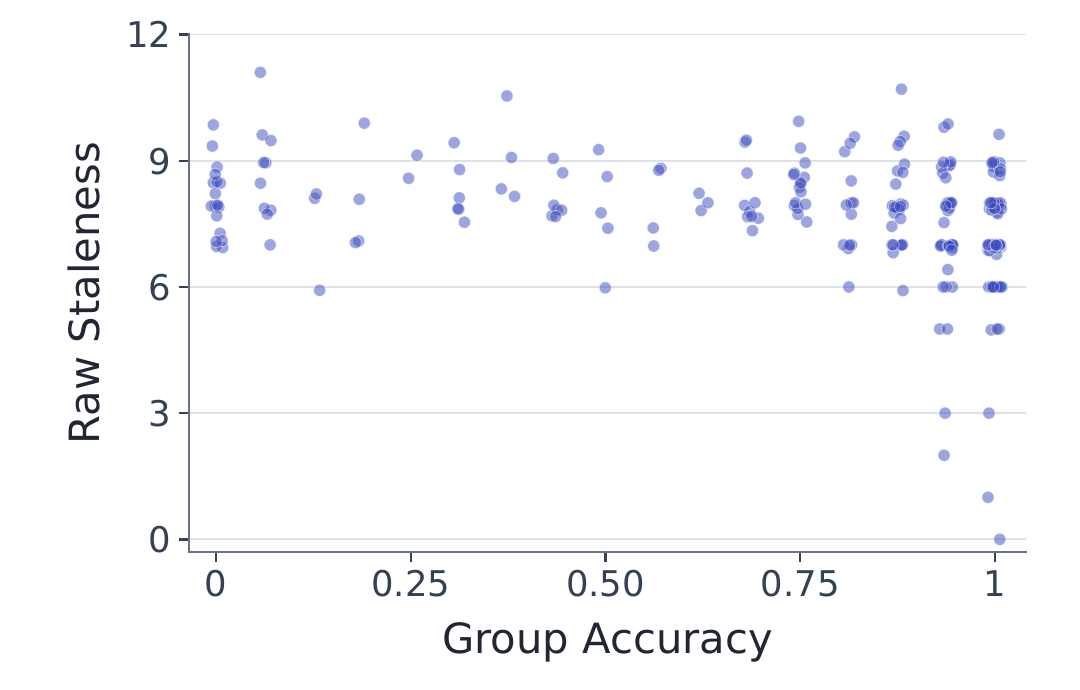}
        \par\smallskip
        {\small (b) Accuracy and raw staleness.}
    \end{minipage}
    \caption{Left: DAPO-MATH-17K training reward for Sync and \aname{} (Raw Staleness). Right: group accuracy versus raw staleness in a fixed-policy ReTool probe.}
    \label{fig:policy-lag-selection-motivation}
\end{figure*}

\subsection{Ablation Study}
\label{app:ablation-study}

\subsubsection{Ablation on Model Scale.}
{Figure~\ref{fig:validation_ablation_panels} compares Async, \aname{} (Raw Staleness), and \aname{} using Qwen3-4B. It reports AIME 2024 and AIME 2025 mean@32 validation accuracy over the first 300 training steps.} 

\begin{figure}[H]
    \centering
    \begin{minipage}[t]{0.485\textwidth}

        \centering
        \includegraphics[width=\linewidth]{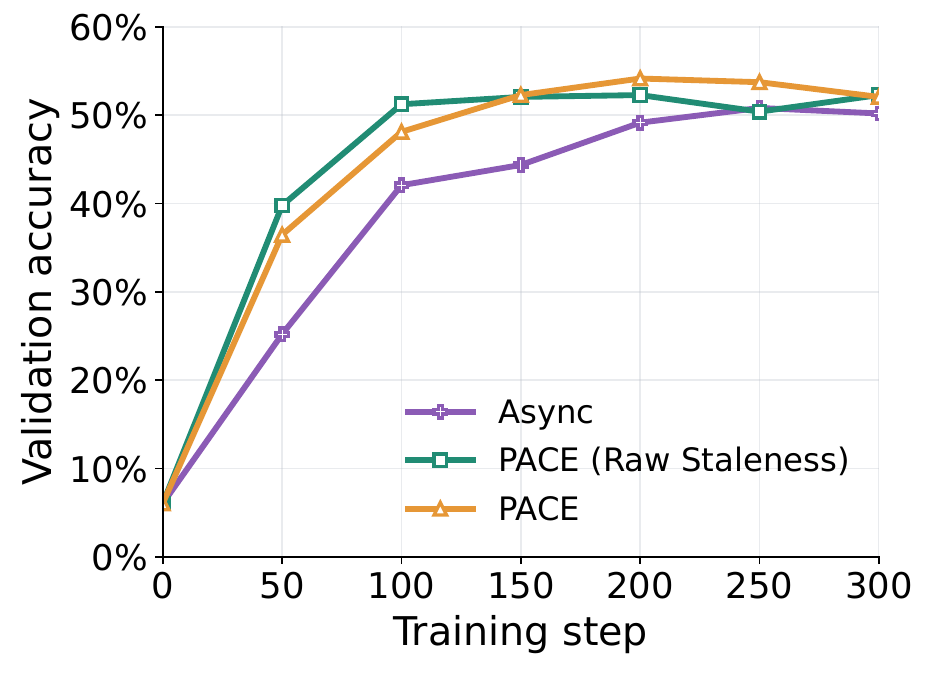}
        \small (a) AIME 2024.
    \end{minipage}
    \hfill
    \begin{minipage}[t]{0.485\textwidth}

        \centering
        \includegraphics[width=\linewidth]{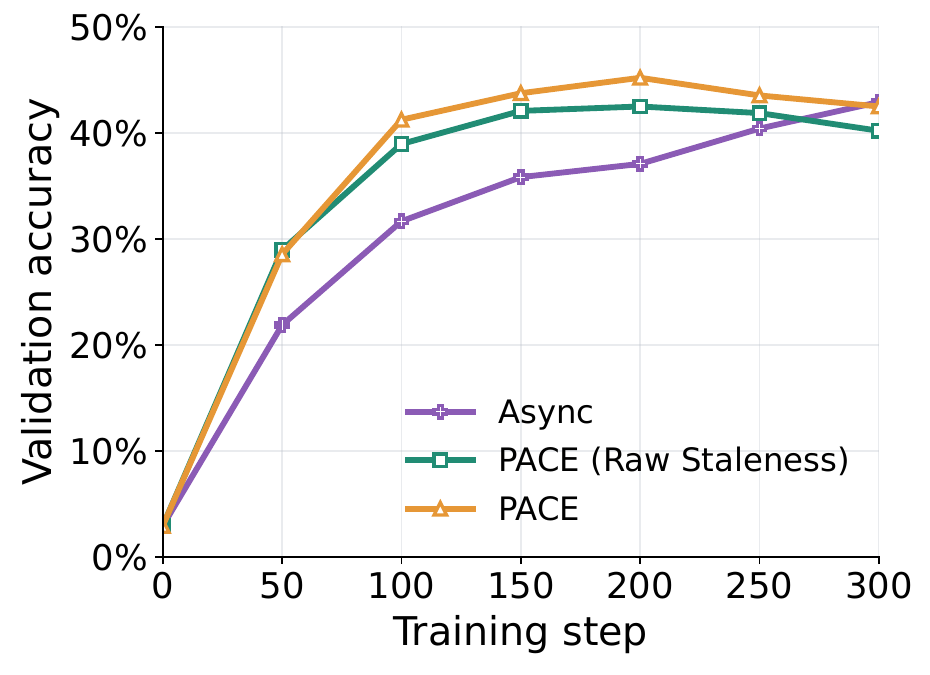}
        \small (b) AIME 2025.
    \end{minipage}
    \caption{Qwen3-4B validation accuracy versus training step on AIME 2024 (left)
    and AIME 2025 (right) over the first 300 training steps. Curves show the local mean@32 for
    Async, \aname{} (Raw Staleness), and \aname{}.
    The common measured initial-model evaluation is reused at step zero as a display reference.}
    \label{fig:validation_ablation_panels}
\end{figure}

\subsubsection{Sensitivity to the Parameter
\texorpdfstring{$\boldsymbol{\gamma}$}{γ}.}
\label{app:gamma_sensitive}
We next vary the sharpening parameter over
$\gamma\in\{1,2,4\}$. Recall that $\gamma=1$ recovers the
unsharpened quantile weight, whereas larger values increase the
polarization between prefixes with low and high forgetting scores.
As shown in Figure~\ref{fig:gamma_ablation_panels}\textbf{(a)},
the AIME 2024 validation curves follow similar optimization
trajectories. Their best mean@32 accuracies are $56.88\%$,
$55.31\%$, and $58.13\%$ for $\gamma=1$, $2$, and $4$,
respectively, giving a maximum difference of only $2.82$ percentage
points. The corresponding training-reward curves are provided in
Appendix Figure~\ref{fig:gamma_ablation_panels}\textbf{(b)}.

These results indicate that \aname{} is not particularly sensitive to
the precise value of $\gamma$ within the tested range. In particular,
changing the polarization strength of the quantile-sharpening function
has only a limited effect on the overall optimization dynamics.
Because $\gamma=4$ achieves the highest observed validation accuracy
without destabilizing training, we use $\gamma=4$ as the default in
all main experiments.

\begin{figure}[H]
    \centering
    \begin{minipage}[t]{0.485\textwidth}

        \centering
        \includegraphics[width=\linewidth]{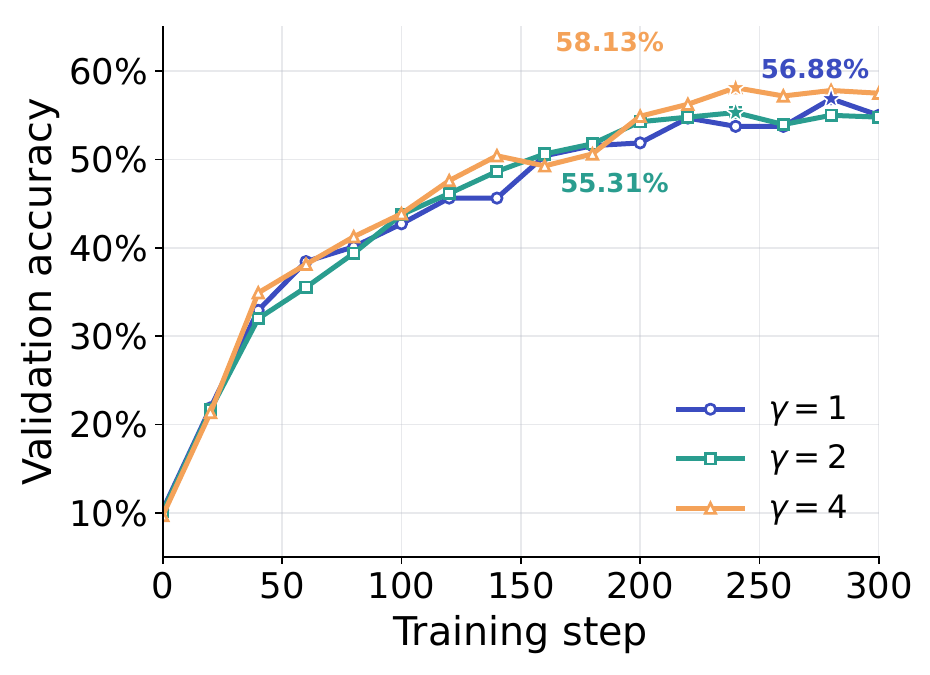}
        \small (a) Validation accuracy.
    \end{minipage}
    \hfill
    \begin{minipage}[t]{0.485\textwidth}

        \centering
        \includegraphics[width=\linewidth]{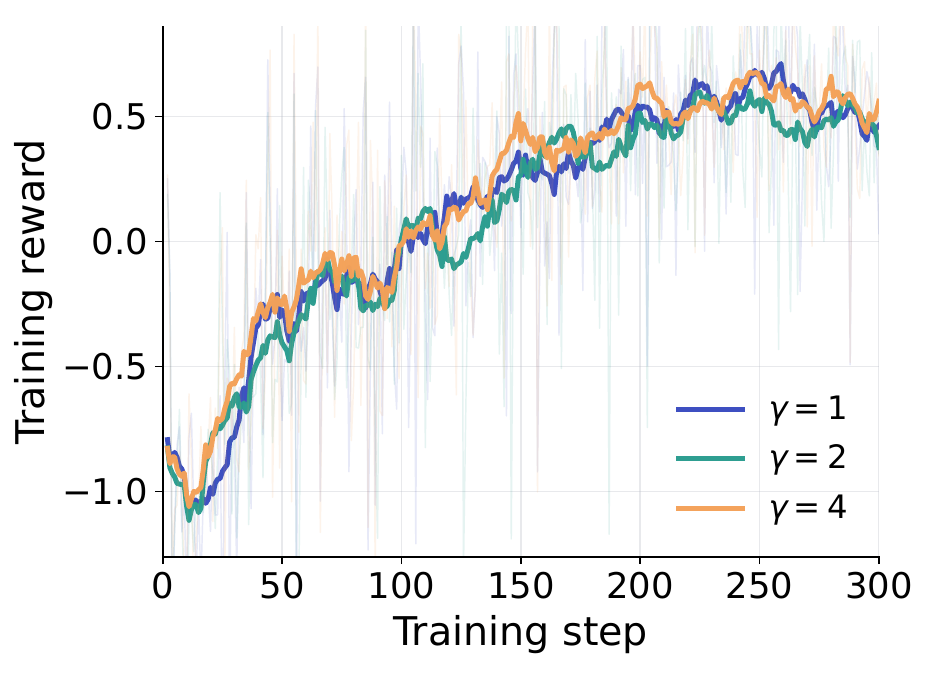}
        \small (b) Training reward.
    \end{minipage}
    \caption{Sensitivity to $\gamma\in\{1,2,4\}$ on Qwen3-8B.
    Left: AIME 2024 mean@32 validation accuracy.
    Right: DAPO-MATH-17K training reward over the first 300 training steps.}
    \label{fig:gamma_ablation_panels}
\end{figure}

\subsubsection{Other RL Algorithms}
\label{app:reinforcepp-ablation}
To examine whether PACE could be applied beyond GRPO, we compare Async with PACE using REINFORCE++ in the ReTool setting with Qwen3-4B-Instruct-2507. Both runs use the same learning rate, minibatch size, rollout group size, and seed, with an asynchronous staleness threshold of eight.  As shown in Figure~\ref{fig:retool-reinforcepp-ablation}, this variant attains higher validation accuracy on both AIME benchmarks later in training. This comparison supports the applicability of pool-aware admission control beyond GRPO;
\begin{figure}[!htbp]
    \centering
    \includegraphics[width=0.9\linewidth]{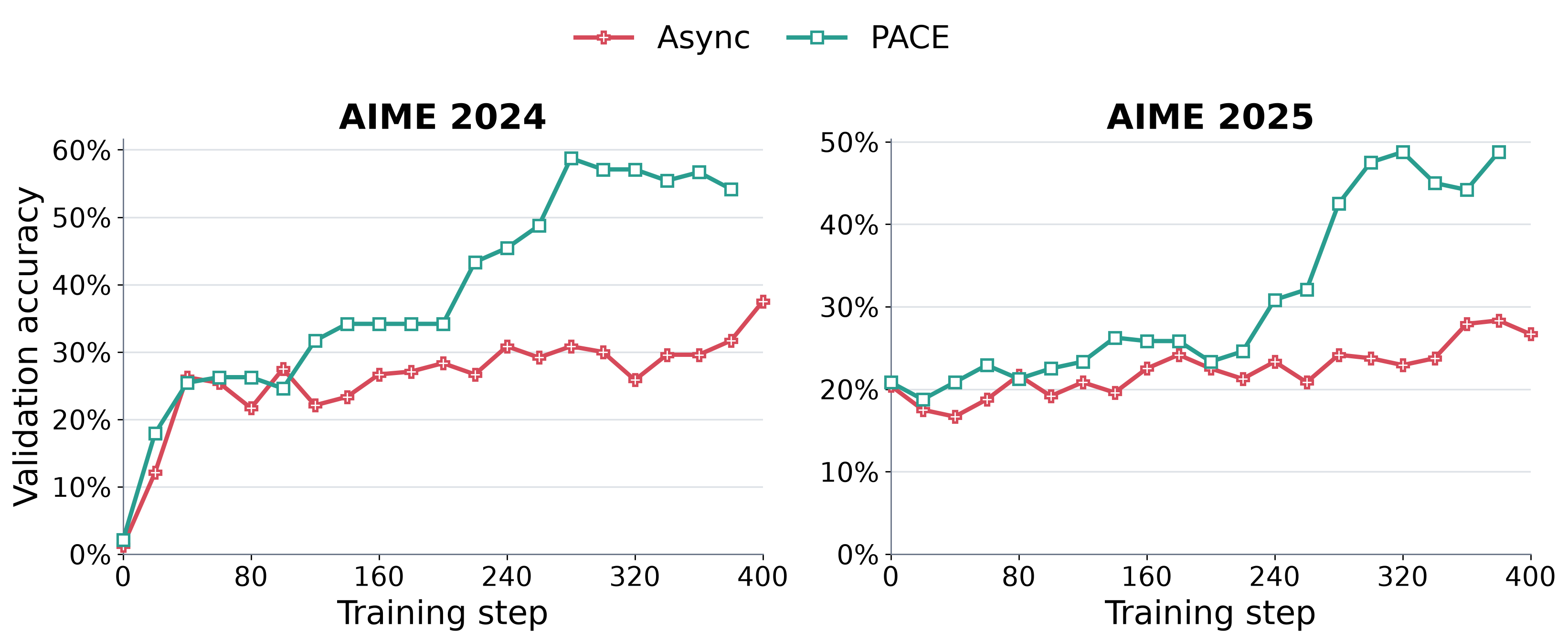}
    \caption{REINFORCE++ on ReTool: AIME 2024 (left) and AIME 2025 (right) validation accuracy versus optimization step.}
    \label{fig:retool-reinforcepp-ablation}
\end{figure}

\FloatBarrier
\section{ReTool System Instruction and Tool Interface}
\label{app:retool-system-instruction}

The training and validation datasets contain a user message, not a
separate hand-written system message. Qwen3-4B-Instruct-2507's tokenizer
constructs the system instruction below from the registered
\texttt{code\_interpreter} schema. We show the resulting prompt,
including the user-message suffix that requests a boxed answer;
\texttt{\{problem\}} is the only substituted field. Line wrapping is
for readability. The same template is used for Sync, Async, and \aname{}.

{The system instruction defines the code-execution interface and the format of tool-call messages. Execution feedback is returned to the model for subsequent reasoning, and the user-message suffix specifies the boxed-answer format used by the verifier.}

\paragraph{Execution and feedback.}
Each tool request runs a self-contained Python program in a stateless
SandboxFusion call, with a 30-second default timeout and a 1024-MB
memory limit. The server permits up to 16 concurrent executions across
rollouts. Tool replies are serialized using \texttt{<tool\_response>}
tags. Replies longer than 256 characters retain their first and last
128 characters, separated by a truncation marker; this limit is in
characters, not tokens. Feedback tokens remain in the context but do
not contribute to the policy loss.

\paragraph{Reward and validation.}
The shared reward checks the final boxed answer, assigning $1$ when
correct. Otherwise it assigns
$\min\{-0.6,-1+0.05(T-2)\}$, where $T$ is the logged total turn count
(initial user message, assistant turns, and tool-response turns).
The reported validation metric instead uses the unshaped binary
correctness indicator, averaged over 32 responses per problem. The
validation sets contain 30 problems each for AIME 2023, AIME 2024,
AIME 2025, BrUMO 2025, and HMMT 2025, and 100 for BeyondAIME.

\begin{table}[!htbp]
\centering
\caption{Prompt template and tool interface used for ReTool
training and evaluation. The placeholder \texttt{\{problem\}}
is replaced with the input question.}
\label{tab:retool-prompt}
\begin{tabular}{@{}p{0.97\linewidth}@{}}
\toprule

\textbf{System message} \\[1mm]
\begin{minipage}[t]{\linewidth}
\begin{lstlisting}[style=retoolprompt,emptylines=0,basicstyle=\fontsize{8}{9}\selectfont\ttfamily\color{PromptText}]
(*@\promptrole{<|im_start|>system}@*)
# Tools

You may call one or more functions to assist with the user query.

You are provided with function signatures within (*@\prompttag{<tools></tools>}@*) XML tags:
(*@\prompttag{<tools>}@*)
{
  "type": "function",
  "function": {
    "name": "code_interpreter",
    "description": "Execute a self-contained Python 3 program for mathematical calculation and verification. The code value must contain raw executable Python source only: do not include Markdown fences, backticks, XML, JSON wrappers, shell commands, or explanatory prose. Explicitly print every result needed in the response. Each call is stateless, so repeat all required imports and definitions. If execution returns an error, correct the code before retrying.",
    "parameters": {
      "type": "object",
      "properties": {
        "code": {
          "type": "string",
          "description": "Raw, self-contained Python 3 source with no Markdown fencing; use print(...) to return results."
        }
      },
      "required": ["code"]
    }
  }
}
(*@\prompttag{</tools>}@*)

For each function call, return a json object with function name and arguments within (*@\prompttag{<tool_call></tool_call>}@*) XML tags:
(*@\prompttag{<tool_call>}@*)
{
  "name": <function-name>,
  "arguments": <args-json-object>
}
(*@\prompttag{</tool_call>}@*)(*@\promptrole{<|im_end|>}@*)
\end{lstlisting}
\end{minipage}
\\[2mm]

\midrule
\textbf{User message} \\[1mm]
\begin{minipage}[t]{\linewidth}
\begin{lstlisting}[style=retoolprompt,emptylines=0,basicstyle=\fontsize{8}{9}\selectfont\ttfamily\color{PromptText}]
(*@\promptrole{<|im_start|>user}@*)
{problem}
The answer format must be: \boxed{'The final answer goes here.'}(*@\promptrole{<|im_end|>}@*)
\end{lstlisting}
\end{minipage}
\\[2mm]

\midrule
\textbf{Assistant generation prefix} \\[1mm]
\begin{minipage}[t]{\linewidth}
\begin{lstlisting}[style=retoolprompt,emptylines=0,basicstyle=\fontsize{8}{9}\selectfont\ttfamily\color{PromptText}]
(*@\promptrole{<|im_start|>assistant}@*)
\end{lstlisting}
\end{minipage}
\\[1mm]

\bottomrule
\end{tabular}
\end{table}

\end{document}